\documentclass[12pt,journal,draftcls,onecolumn]{IEEEtran}

\usepackage{amsthm}
\usepackage{amsmath,amsfonts,mathrsfs}
\usepackage{algorithmic}
\usepackage{algorithm}
\usepackage{array}
\usepackage[caption=false,font=normalsize,labelfont=sf,textfont=sf]{subfig}
\usepackage{textcomp}
\usepackage{stfloats}
\usepackage{url}
\usepackage{hyperref}
\usepackage{verbatim}
\usepackage{graphicx}
\usepackage{cite}

\newtheorem{theorem}{Theorem}

\newtheorem{corollary}{Corollary}
\newtheorem{lemma}{Lemma}
\newtheorem{proposition}{Proposition}

\newtheorem{remark}{Remark}

\renewcommand{\P}{\mathbb{P}}

\newcommand{\R}{\mathbb{R}}
\newcommand{\E}{\mathbb{E}}

\newcommand{\Cov}{\operatorname{Cov}}
\newcommand{\Var}{\operatorname{Var}}

\IEEEoverridecommandlockouts
\begin{document}
\title{On optimal solutions of classical and sliced Wasserstein GANs with non-Gaussian data}

\author{
\IEEEauthorblockN{Yu-Jui Huang$\dag$, Hsin-Hua Shen*, Yu-Chih Huang$\S$, Wan-Yi Lin $\ddag$, and Shih-Chun Lin*}

\IEEEauthorblockA{
$\dag$ Dept. of Applied Math., Univ. of Colorado, Boulder, CO 80309, USA,
\url{yujui.huang@colorado.edu}\\
* Department of EE and GICE, National Taiwan University, Taipei, Taiwan, \url{sclin2@ntu.edu.tw}
\\
$\S$ Institute of Communications Engineering, National Yang Ming Chiao Tung University, HsinChu, Taiwan, \url{jerryhuang@nctu.edu.tw}
\\
$\ddag$ Bosch Center for AI, USA, \url{wan-yi.lin@us.bosch.com}
}
\thanks{
The first two authors have equal contributions and are ordered alphabetically. An earlier version of this paper was presented in part at the 2023 IEEE International Symposium on Information Theory (ISIT)\cite{10206785}. This work was supported in part by NSF under Grant DMS-2109002; in part by the National Science and Technology Council, Taiwan, under Grant 111-2221-E-002-099-MY2, 111-3114-E-002-001 and  MOST 111-2221-E-A49 -069 -MY3. 
}}

\maketitle

\begin{abstract}
The generative adversarial network (GAN) aims to approximate an unknown distribution via a parameterized neural network (NN). While GANs have been widely applied in reinforcement and semi-supervised learning as well as computer vision tasks, selecting their parameters often needs an exhaustive search, and only a few selection methods have been proven to be theoretically optimal. One of the most promising GAN variants is the Wasserstein GAN (WGAN). Prior work on optimal parameters for population WGAN is limited to the linear-quadratic-Gaussian (LQG) setting, where the generator NN is linear, and the data is Gaussian. In this paper, we focus on the characterization of optimal solutions of population WGAN beyond the LQG setting. As a basic result, closed-form optimal parameters for one-dimensional WGAN are derived when the NN has non-linear activation functions, and the data is non-Gaussian. For high-dimensional data, we adopt the sliced Wasserstein framework and show that the linear generator can be asymptotically optimal. Moreover, the original sliced WGAN only constrains the projected data marginal instead of the whole one in classical WGAN,  and thus, we propose another new unprojected sliced WGAN and identify its asymptotic optimality.  Empirical studies show that compared to the celebrated r-principal component analysis (r-PCA) solution, which has cubic complexity to the data dimension, our generator for sliced WGAN can achieve better performance with only linear complexity.
\end{abstract}

\section{Introduction}
Generative adversarial networks (GANs) are machine learning frameworks put forth by Goodfellow {\it et al.} \cite{Goodfellow_NIPS2014}. A GAN aims to learn an unknown distribution from training data via two competing components, namely the generator and the discriminator. The former tries to mimic the distribution of training data while the latter discriminates between true data and generated data. This framework formulates a minimax problem, where the discriminator aims to maximize the probability of differentiating true data from generated ones, while the generator seeks to minimize this probability \cite{Goodfellow_NIPS2014}. Besides computer vision tasks, applications of GANs to communication systems have also garnered significant attention. For example, GANs have been applied to autonomous wireless channel modeling \cite{CommGANchannelmodel,lin2023information} and mmWave MIMO channel estimation \cite{9953094}.

Traditionally, both the generator and discriminator in GAN are approximated by neural networks (NNs) \cite{Goodfellow_NIPS2014, USC_NIPS18}. Without the NN restrictions and under an optimal unconstrained discriminator, the minimax problem associated with a GAN becomes the minimization of the Jensen-Shannon divergence (JSD) between the distributions of true and generated data. However, due to the nature of this minimax game, the GAN suffers from several problems, including vanishing gradient and mode collapse. Many variants of GAN have been proposed to solve these problems, and one of the most promising variants is the Wasserstein GAN (WGAN) \cite{WGAN}, which replaces the JSD by the Wasserstein distance from the optimal transport theory \cite{peyre2019computational}. The WGAN is differentiable with respect to the generator parameters almost everywhere, which benefits the convergence of stochastic gradient descent (SGD) widely adopted for training NNs \cite{WGAN}. 

Despite the substantial success of applying WGAN to learn distributions in real applications, there are only a few GAN parameter selection algorithms proven to be theoretically optimal \cite{Tse20, OzgurISIT21}, which limits the development of GAN beyond heuristic methods in \cite{Goodfellow_NIPS2014, USC_NIPS18}. This lack of rigorous analysis also restricts the evaluation of GANs' performance to subjective terms. One exception is \cite{Tse20} where Feizi {\it et al.} attempted to theoretically understand WGANs in a simple linear quadratic Gaussian (LQG) setting. In this benchmark setting, the synthetic data is generated by a Gaussian distribution, the generator NN is restricted to be linear, and the loss function is quadratic. It is shown in \cite[Theorem 1]{Tse20} that for this simple setting, the optimal GAN solution happens to be the principal component analysis (PCA) solution.  Regularized versions of GANs are also considered and analyzed \cite{USC_NIPS18}. Optimal WGAN solutions under LQG settings, with additional entropic and Sinkhorn regularizers, are also studied \cite{OzgurISIT21}. 

In this paper, we aim to solve WGANs beyond the LQG setting analytically. As described in Sec. \ref{sec:Problem}, our setting allows non-Gaussian data distribution and non-linear generators constructed from sigmoid and ReLU functions, and is therefore more general than \cite{Tse20, OzgurISIT21,bailey2018size}. Also, we attempt to exactly solve WGANs rather than their regularized versions as in \cite{OzgurISIT21}. All of these make our problem much more challenging---after all, even for the inner discriminator problem, which is an optimal transport problem, the solution in most cases is numerically approximated but not analytically characterized  \cite{peyre2019computational}. To overcome the challenge, we first focus on one-dimensional data and generator in Sec. \ref{subsec:reuslt_d_1}, where we provide closed-form solutions for optimal generators in one-dimensional WGANs defined in Sec. \ref{sec:Problem}.

For high-dimensional WGANs, we adopt the sliced WGAN framework, which provides a computationally efficient alternative by reducing high-dimensional distributions to one-dimensional projections and averaging Wasserstein distances of projected distributions \cite{Deshpande_CVPR2018, Kolouri_NIPS19,nadjahi2021fast,10657311}. Also, we define a revised sliced WGAN with unprojected data distribution in Sec. \ref{sec:Problem}. We provide closed-form solutions for asymptotically optimal generators in high-dimensional sliced WGAN in \ref{sec:SlicedWGAN}. The proofs are presented in Sec. \ref{sec:proof}, where we leverage results in \cite{AP03} to solve the inner discriminator problem in closed form, which significantly simplifies the necessary conditions for optimal WGAN parameters. Moreover, our closed-form solutions do not need any training for the discriminator as \cite{chen2019gradual} and hence provide additional benefit for training WGAN with a decentralized system \cite{mcmahan2017federated}. Empirical studies in Sec. \ref{sec:simu} show that our closed-form one-dimensional WGAN parameters have good convergence behavior with synthetic data under both Gaussian and Laplace distributions. High-dimensional empirical studies also show that our closed-form sliced WGAN generator can achieve similar or better sliced-Wasserstein distance performance as r-PCA solutions in \cite{Tse20}\cite{Goodfellow-et-al-2016}, while reducing computational complexity from cubic to linear in data dimension. Compared with our previous work \cite{10206785}, the asymptotic optimality of linear generators in Theorems \ref{optimalslicedWasserstein} and \ref{thm:unprojectedsliced} is new, as is the proof of Theorem \ref{thm:theta's ReLU}, which is not presented in \cite{10206785}. 

Throughout the paper, we adopt the following notational conventions. The operator $\mathrm{diag}(\cdot):\R^{d\times d}\to\R^d$ extracts the diagonal entries of a $d\times d$ matrix and returns them as a $d$-dimensional vector. We let $\mathbf{1}\in\R^d$ denote the vector whose components are all equal to one, and $e_i\in\R^d$ denote the $i$-th standard unit vector, with its $i$-th component equal to $1$ and all others equal to $0$. If $\mu$ is a probability measure on $\mathcal{X}$, and $T$ is a mapping $\mathcal{X}\rightarrow\mathcal{Y}$, then $T\#\mu$ stands for the push-forward of $\mu$ by $T$.

\section{Problem Formulation}\label{sec:Problem}
For the WGAN considered in this paper, the generator function is denoted by $G_\Theta(\cdot)$, where $\Theta$ are the generator NN's parameters (weights). 
Let $q\in \{1,2\}$ represent the order of the Wasserstein distance. In a $q^{\mathrm{th}}$-order WGAN setting, one aims to solve
\begin{equation} \label{eq_transportGANE}
\min_{\Theta} \left(\inf_{\pi\in \Pi(\mu,\nu^\Theta)} \E_\pi[\|X-G_\Theta(Z)\|^q]\right)^{1/q}
\end{equation}
for an optimal parameter $\Theta$, where $\|\cdot\|$ denotes the $\ell_2$-norm in $\R^d$, $\mu$ and $\nu^\Theta$ are probability measures on $\R^d$, generated by the data $X$ and the generator output $G_\Theta(Z)$ for a given $\Theta$ and Gaussian input $Z$, respectively. Also, $\Pi(\mu,\nu^\Theta)$ is the set of probability measures on $\R^d\times\R^d$ with marginal constraints $\mu$ and $\nu^\Theta$, respectively. That is, a $\pi\in\Pi$ satisfies the marginal conditions: $(\text{proj}^1_{\pi}:\R^d\times\R^d\rightarrow\R^d)\#\pi=\mu$, and $(\text{proj}^2_{\pi}:\R^d\times\R^d\rightarrow\R^d)\#\pi=\nu^\Theta$, where $\text{proj}^1_\pi$ denotes the projection onto the first coordinate of $\pi$, and $\text{proj}^2_\pi$ denotes the projection onto the second coordinate of $\pi$. The marginal constraints $\mu$ and $\nu^\Theta$ satisfy
$
\int_{\R^d} \|x\|^q d\mu(x)<\infty,
$ and $\int_{\R^d} \|y\|^q d\nu^\Theta(y)<\infty$. The WGAN problem described in \eqref{eq_transportGANE} can be equivalently written as $\min_{\Theta} (\E_\mu [P(X,\Theta)])^{1/q}$ where the inner-discriminator problem is defined as
\begin{equation}\label{P}
\E_\mu [P(X,\Theta)]:=\inf_{\pi\in \Pi(\mu,\nu^\Theta)} \E_\pi[\|X-Y\|^q].
\end{equation}
Note that \eqref{P} belongs to the family of optimal transport problems with
$q^{\mathrm{th}}$-order Wasserstein distance, $q\in\{1,2\}$ \cite[Proposition
2.2]{peyre2019computational}.
In this paper, we consider the following  activation functions: 1) the linear function; 2) the rectified linear unit (ReLU) function $\max(0; z)$; and 3) the sigmoid function
$
1/(1+\exp(-z)).
$

Following \cite{Goodfellow_NIPS2014}\cite{Tse20}, \eqref{eq_transportGANE} is a population WGAN problem, where $\Theta$ can be optimized with the true data distribution $\mu$, and studying the theoretical population WGAN can lead to practical design with large enough empirical data. In Sec. \ref{subsec:reuslt_d_1}, we first theoretically solve the population WGAN and obtain the closed-form solutions under $d=1$. 

Note that for WGAN \eqref{eq_transportGANE}, the closed-form solution was only obtained in the LQG setting \cite{Tse20}. In the LQG setting, the $d$-dimensional synthetic data is generated by a Gaussian distribution, i.e., $X\sim\mathcal{N}(0,\mathbf{K})$, the generator is restricted to be a linear generator of the form
\begin{equation} \label{eq_highdim}
\Theta Z,\quad \Theta \in \R^{d\times r},
\end{equation}
with $Z\sim \mathcal{N}(0, \mathbf{I}_r)$ a Gaussian vector, and the loss function is quadratic (i.e., $q=2$). 

Beyond the LQG setting, we show how to apply our $d=1$ results for \eqref{eq_transportGANE} to cope with higher-dimensional data $d>1$ by adopting the sliced Wasserstein distance, which offers a computationally efficient alternative to the WGAN \cite{Deshpande_CVPR2018, Kolouri_NIPS19,nadjahi2021fast}. Consider the unit sphere
\begin{equation}\label{unitsphere}
\Omega=\{\omega\in\mathbb{R}^d:~\|\omega\|=1\},
\end{equation}
which contains all the directions in $\mathbb{R}^d$. In sliced WGAN, we project both the data and the generator's output onto $\R$ through a random direction $\omega\in\Omega$  and compute the Wasserstein distance between the projected one-dimensional distributions $\mu_\omega$ and $\nu^\Theta_\omega$, i.e., $\omega^T X$ and $\omega^TG_\Theta(Z)$ respectively follow distributions $\mu_\omega$ and $\nu^\Theta_\omega$. By replacing the Wasserstein distance of the inner-discriminator problem \eqref{P} with the sliced one, 
\begin{equation}\label{slicedinner}
\E_{\mu_\omega} [P(X,\Theta)]:=\int_{\omega\in\Omega}\inf_{\pi_\omega\in \Pi_\omega(\mu_\omega,\nu^\Theta_\omega)} \mathbb{E}_{\pi_\omega}[|\omega^T X - \omega^TG_\Theta(Z)|^q] d\omega ,
\end{equation}
\noindent the $q^{\mathrm{th}}$-order sliced (population) WGAN is
\begin{align} \label{eq_SlicedWGAN}
\min_\Theta \E_{\mu_\omega} [P(X,\Theta)]^{\frac{1}{q}} = \min_\Theta \left(\int_{\omega\in\Omega}\inf_{\pi_\omega\in \Pi_\omega(\mu_\omega,\nu^\Theta_\omega)} \mathbb{E}_{\pi_\omega}[|\omega^T X - \omega^TG_\Theta (Z)|^q] d\omega \right)^{\frac{1}{q}}, 
\end{align}
where $\Pi_\omega(\mu_\omega,\nu^\Theta_\omega)$ is the set of probability measures on $\R\times\R$ with marginal constraints $\mu_\omega$ and $\nu^\Theta_\omega$, generated by the projected data $\omega^T X$ and the projected generator output $\omega^T G_\Theta (Z)$, respectively. That is, a $\pi_\omega\in\Pi_\omega$ satisfies the marginal conditions: $(\text{proj}^1_{\pi_\omega}:\R\times\R\rightarrow\R)\#\pi_\omega=\mu_\omega$, and $(\text{proj}^2_{\pi_\omega}:\R\times\R\rightarrow\R)\#\pi_\omega=\nu^\Theta_\omega$.

Moreover, though widely adopted in \cite{Deshpande_CVPR2018, Kolouri_NIPS19,nadjahi2021fast}, the joint PDF set $\Pi_\omega$ of the inner-discriminator problem \eqref{slicedinner} for the sliced WGAN \eqref{eq_SlicedWGAN} is constrained by projected marginal $\mu_\omega$ on $\omega^TX$. However, since projections of distributions may not be one-to-one, the optimal $\pi_\omega$ may correspond to a data distribution $\tilde{\mu}$ on $\R^d $ that is different from the original $\mu$, while $\tilde{\mu}$ and $\mu$ have the same $\mu_\omega$ after projection. To avoid this issue, it is reasonable to consider the set of probability measures $\Pi'(\mu,\nu^\Theta_\omega)$ on $\R^d \times \R$ with marginal constraints $\mu$ and $\nu^\Theta_\omega$, generated by the unprojected data $X$ and the projected generator output $\omega^T G_\Theta (Z)$, respectively. That is, a $\pi'\in\Pi'$ satisfies the marginal conditions: $(\text{proj}^1_{\pi'}:\R^d\times\R\rightarrow\R^d)\#\pi'=\mu$, and $(\text{proj}^2_{\pi'}:\R^d\times\R\rightarrow\R)\#\pi'=\nu^\Theta_\omega$. The new sliced WGAN problem is
\begin{equation} \label{eq_SlicedWGANwithdata}
\min_\Theta \left(\int_{\omega\in\Omega}\inf_{\pi'\in \Pi'(\mu,\nu^\Theta_\omega)} \mathbb{E}_{\pi'}[|\omega^T X - \omega^TG_\Theta(Z)|^q] d\omega \right)^{\frac{1}{q}}.
\end{equation}
Comparing \eqref{eq_SlicedWGANwithdata} with \eqref{eq_SlicedWGAN}, the optimal transport function for the inner discriminator problem of \eqref{eq_SlicedWGANwithdata} is different from that of \eqref{eq_SlicedWGAN} due to the definitions of sets $\Pi'$ and $\Pi_\omega$. The main difference is on their first marginal condition. For a $\pi_\omega\in\Pi_\omega$, it satisfies $(\text{proj}^1_{\pi_\omega}:\R\times\R\rightarrow\R)\#\pi_\omega=\mu_\omega$, where $\mu_\omega$ is the projected data distribution, while a $\pi'\in\Pi'$ satisfies $(\text{proj}^1_{\pi'}:\R^d\times\R\rightarrow\R^d)\#\pi'=\mu$, where $\mu$ is the original data distribution.

Note that for all generator functions, the corresponding optimal values from original \eqref{eq_SlicedWGAN} and the new \eqref{eq_SlicedWGANwithdata} with unprojected $\mu$ are non-negative. Thus, we say that a sequence of generator functions is asymptotically optimal for \eqref{eq_SlicedWGAN} if the following holds: there exists a sequence of generator functions $G_\Theta(\cdot)\in\R^d$ such that the corresponding values of \eqref{eq_SlicedWGAN} go to zero as $d\rightarrow\infty$. The same definition applies to \eqref{eq_SlicedWGANwithdata}. In Sec. \ref{sec:SlicedWGAN}, we identify the asymptotically optimal generator for both original and unprojected sliced WGAN problems \eqref{eq_SlicedWGAN} and \eqref{eq_SlicedWGANwithdata}.

\begin{table*}[!t]
\caption{Comparison to \cite{Tse20, OzgurISIT21,bailey2018size}, where $q$ is WGAN order, and $d$ is data dimension \label{tab:table1}}
\centering
\begin{tabular}{|c||c|c|c|c|}
\hline
Paper/Theorem & Non-linear generator & Closed-form solution & Continuous data distribution & WGAN\\
\hline
\cite{Tse20} & X & V & Gaussian & $q=2$  \\
\hline
\cite{OzgurISIT21} & X & V & Gaussian & $q=2$ regularized \\
\hline
\cite{bailey2018size} & V & X & Gaussian/Uniform & $q=1$\\
\hline
Theorem \ref{thm:theta's ReLU} & V & V & No limit & $q=2, d=1$\\
\hline
Theorem \ref{optimalslicedWasserstein} & V & V &  No limit & $q=2$ sliced\\
\hline
Theorem \ref{thm:unprojectedsliced} & V & V & No limit & $q=2$ unprojected sliced \\
\hline
\end{tabular}
\end{table*}

\section{Main Results}\label{sec:Main}
In Sec. \ref{subsec:reuslt_d_1}, we analytically solve WGAN \eqref{eq_transportGANE} with the non-linear generator and non-Gaussian data under $d=1$ when $q=2$. With large $d$ in Sec. \ref{sec:SlicedWGAN}, we show that the linear generators can be asymptotically optimal for both original and unprojected sliced WGANs \eqref{eq_SlicedWGAN} and \eqref{eq_SlicedWGANwithdata} with non-Gaussian data by solving corresponding upper bounds. We summarize and compare our results with the prior works \cite{Tse20, OzgurISIT21, bailey2018size} in Table~\ref{tab:table1}.

\subsection{Basic results for classical WGAN under one dimension}\label{subsec:reuslt_d_1}
First, we consider the quadratic case $q=2$ with a non-linear generator
\begin{equation}\label{G}
G_\Theta(Z) = \theta_1+\theta_2 h(Z),
\end{equation}
where $h:\R\to\R$ and $Z\sim N(0,1)$, also $(\theta_1,\theta_2)\in\R\times\R$ are parameters of the generator NN to be selected. Let $\Psi$ denote the cumulative distribution function (CDF) of $h(Z)$, for any continuous data distribution $\mu$, our closed-form WGAN parameters are as follows:

\begin{proposition}\label{prop:theta's}
Assume CDF $F_\mu$ of $\mu$ and CDF $\Psi$ of $h(Z)$ in \eqref{G} are continuous and strictly increasing, and variance $\Var(h(Z))>0$. If
\begin{equation}\label{Cov>0}
\Cov\left(X, \Psi^{-1}(F_\mu(X))+\Psi^{-1}(1-F_\mu(X))\right)\ge0,
\end{equation}
the population WGAN \eqref{eq_transportGANE} with $q=2,d=1$ has a unique minimizer for $(\theta_1,\theta_2)\in\R\times\R$ as
\begin{equation}\label{theta's}
\begin{split}
\theta^*_2 &=  \frac{\Cov\left(X, \Psi^{-1}(F_\mu(X))\right)}{\Var(h(Z)) } \ge 0,\\
\theta^*_1 &=  \E_\mu[X] -  \theta^*_2 \E_g \left[h(Z)\right];
\end{split}
\end{equation}
if \eqref{Cov>0} is not met, $(\theta_1^*,\theta_2^*)$ is given by replacing $\theta^*_2$ in \eqref{theta's} by
\begin{equation}\label{theta's''}
\begin{split}
\theta^*_2 &=  \frac{\Cov\left(X, \Psi^{-1}(1-F_\mu(X))\right) }{\Var(h(Z)) }\le 0,
\end{split}
\end{equation}
where $\E_g$ is taking expectation over Gaussian $Z \sim \mathcal{N}(0,1)$.
\end{proposition}
\noindent We identify properties of data PDF \eqref{Cov>0} such that the optimal parameters $\theta_2$ have different forms \eqref{theta's} and \eqref{theta's''}. The proof for proposition \ref{prop:theta's} is similar to that for the upcoming Theorem \eqref{thm:theta's ReLU} and given in Appendix \ref{proofSimoid}.

Let us now look at some specific cases of $h(z)$. For the sigmoid function $h(z)$, recall that the logit function
$
\text{logit}(p) :=  \ln\left(p/(1-p)\right), p\in(0,1)
$
which is the inverse function $h^{-1}(z)$. The random variable $h(Z)$ has a logit-normal distribution, i.e. $\text{logit}(h(Z))$ is normally distributed, with CDF
\[
\Psi(v) = \frac12 \left(1+\text{erf}\left(\frac{\text{logit}(v)}{\sqrt{2}}\right)\right)\quad \hbox{for}\ v\in(0,1).
\]

For the ReLU function,  the CDF of $h(Z) = \max\{Z,0\}$ is given by
\begin{equation}\label{cdf ReLU}
\Psi(v) = \Phi(v) \cdot 1_{\{v\ge 0\}}  ,
\end{equation}
where $\Phi$ is the CDF of Gaussian $N(0,1)$. However, now $\Psi$ has a jump from 0 to $1/2$ at $v=0$ and does not meet the setting of Proposition \ref{prop:theta's} since it is neither continuous nor strictly increasing. We need the following modification. 

\begin{theorem}\label{thm:theta's ReLU}
Assume $\mu$ has the setting as in Proposition \ref{prop:theta's} and $\Psi$ is given by \eqref{cdf ReLU}. If
\begin{align}
 \Cov\left(X, \Phi^{-1}(F_\mu(X)) 1_{\{F_\mu(X)>1/2\}} \right) \ge  \Cov\left(X, \Phi^{-1}(F_\mu(X)) 1_{\{F_\mu(X)\le1/2\}} \right), \label{Cov>0 ReLU}
\end{align}
the population WGAN \eqref{eq_transportGANE} with $q=2,d=1$ has a unique minimizer for $(\theta_1,\theta_2)\in\R\times\R$ as
\begin{equation}\label{theta's ReLu}
\begin{split}
\theta^*_2 &=   \frac{2\pi}{\pi-1} \Cov\left(X, \Phi^{-1}(F_\mu(X)) 1_{\{F_\mu(X)> 1/2\}}
\right)  \\
\theta^*_1 &=  \E[X] -  \theta^*_2/\sqrt{2\pi};
\end{split}
\end{equation}
if \eqref{Cov>0 ReLU} is not met,  $(\theta_1^*,\theta_2^*)$ is given by replacing $\theta^*_2$ in \eqref{theta's ReLu} by
\begin{equation}\label{theta's ReLu'}
\theta^*_2 =   -\frac{2\pi}{\pi-1} \Cov\left(X, \Phi^{-1}(F_\mu(X)) 1_{\{F_\mu(X)\le1/2\}} \right).
\end{equation}
\end{theorem}

To solve the inner discriminator problem \eqref{P}, we break \eqref{eq_transportGANE} down into two sub-problems depending on the sign of $\theta_2$, i.e., $\min_{\theta_1,\theta_2\in\R} \E_\mu [P(X,\theta_1,\theta_2)]$ equals to
\begin{equation} \label{eqW2sub}
\min \! \left( \! \min_{\theta_1\in\R,\theta_2\ge 0}  \E_\mu [P(X,\theta_1,\theta_2)],\! \min_{\theta_1\in\R,\theta_2\le 0}  \E_\mu [P(X,\theta_1,\theta_2)] \! \! \right)
\end{equation}
Since the CDF of the ReLU function has a jump, the transfer function will be a piecewise function depending on the CDF in each subproblem. By solving Karush–Kuhn–Tucker (KKT) conditions of each sub-problem, the solution of the first sub-problem is indeed \eqref{theta's ReLu} while that of the second sub-problem is \eqref{theta's ReLu'}. We derive the closed-form primal optimal solution by eliminating the Lagrange multipliers, thereby reducing the KKT equations to a set of conditions involving only the primal variables. The condition \eqref{Cov>0 ReLU} is obtained by comparing the values of the two subproblems. The full proof for Theorem \ref{thm:theta's ReLU} is in Sec. \ref{proofReLU}.

For the linear generator $h(Z)=Z$ as \cite[eqn. (27)]{USC_NIPS18}\cite{Tse20}\cite{OzgurISIT21}, one can simplify the results in Proposition \ref{prop:theta's} and also get alternative closed-form formula as follows.

\begin{corollary}\label{prop:theta's Linear}
Assume $\mu$ has the setting as in Proposition \ref{prop:theta's} and consider the linear case $h(Z)=Z$ in \eqref{G}. Then population WGAN \eqref{eq_transportGANE} has a unique minimizer for $(\theta_1,\theta_2)\in\R\times\R$ as
\begin{equation}
\theta^*_1 =  \E_\mu[X]\quad \hbox{and}\quad \theta^*_2 = \E_\mu [X \cdot \Phi^{-1}(F_\mu(X))], \label{eq_W2thetaOpt1}
\end{equation}
also equivalently
\begin{equation} \label{optParaLinear1}
\theta^*_2=\E_g[F_\mu^{-1}(\Phi(Z)) \cdot Z],
\end{equation}
\end{corollary}
\noindent The proof of Corollary \ref{prop:theta's Linear} is in Appendix \ref{proof:theta's Linear}. 

Here we briefly compare our proposed WGAN solutions with the results for the LQG setting in \cite[Theorem 1]{Tse20}.  Unlike \cite{Tse20}, our results can deal with non-Gaussian data distribution. Moreover, we can recover the result in \cite{Tse20} when $q=2,d=1$. Indeed, if $X\sim\mathcal{N}(0,\sigma^2)$, by looking into \eqref{eq_W2thetaOpt1}, we have
    $
        \Phi^{-1}(F_\mu (X))= \Phi^{-1}(\Phi(X/\sigma))= X/\sigma.
    $
    Plugging this into our solution in \eqref{eq_W2thetaOpt1} shows that $\theta_1^*=0$ and $\theta_2^*=\mathbb{E}[X\frac{X}{\sigma}]=\sigma$, coinciding with the result in \cite{Tse20} for $d=r=1$. 
We emphasize that neither our results nor those in \cite{Tse20} subsume the other as a special case, as \cite{Tse20} considers general dimension $d$ while our work is not restricted to Gaussian data distribution. Since WGAN output is already Gaussian $\Theta Z\sim\mathcal{N}(0,\mathbf{K}')$, it is shown in \cite[Theorem 1]{Tse20} that for this benchmark, the optimal WGAN solution happens to be the $r$-PCA solution of Gaussian $X$. That is, the optimal generator matrix $\Theta$ fulfills that $\mathbf{K}'=\Theta\Theta^T$ is a rank $r$ matrix and $\mathbf{K}'$ and $\mathbf{K}$ share the same largest $r$ eigenvalues and the corresponding eigenvectors.

Finally, we present our results for non-linear generators and first-order Wasserstein distance ($q=1$). We have
\begin{corollary} \label{prop:W1d1} Following the settings in Proposition \ref{prop:theta's}, the minimizer $(\theta^*_1, \theta^*_2)$ of the population WGAN \eqref{eq_transportGANE} with $q=1,d=1$ meet the following necessary conditions
\begin{align}
&\E_\mu \! \! \left[ \text{sign}\Big(\theta^*_1 + \theta^*_2 \Psi^{-1}(F_\mu(X))-X\Big)  \right] =0, \notag \\
&\E_\mu \! \! \left[ \text{sign}\!\Big(\theta^*_1 + \theta^*_2 \Psi^{-1}(F_\mu(X))-X\Big) \Psi^{-1}(F_\mu(X))\right]\! \!=\! \!0, \label{eq_W1condition}
\end{align}
when $\theta^*_2>0$, where $\text{sign}(x)=1,0,-1$ for $x>0$, $x=0$, $x<0$, respectively.
\end{corollary}
The proof is omitted. Note that when $X \sim \mathcal{N}(\mu, \sigma^2)$ then $F_\mu(X)=\Phi((X-\mu)/\sigma)$, It can be checked that for the linear case as Corollary \ref{prop:theta's Linear}, the optimal  $(\theta^*_1,\theta^*_2)=(\mu, \sigma)$ for $q=2$ also meets \eqref{eq_W1condition} and is at least a local optimum for $q=1$.

\subsection{Asymptotic optimal results for sliced WGANs under high dimension} \label{sec:SlicedWGAN}
We first focus on the asymptotic optimal generator of \eqref{eq_SlicedWGAN} as defined in Sec. \ref{sec:Problem}. Our main result shows that the linear generator can be asymptotically optimal for the original sliced WGAN \eqref{eq_SlicedWGAN} when $q=2$.
\begin{theorem} \label{optimalslicedWasserstein}
For the second-order sliced WGAN \eqref{eq_SlicedWGAN}, if data $X$ satisfy the weak dependence condition in \cite[Definition 1]{nadjahi2021fast}, the linear generator \eqref{eq_highdim} is asymptotically optimal with the parameters
\begin{equation}\label{optimallinearparameter}
\Theta^*=\mathbf{U}\Lambda \mathbf{V}^T \text{ where } \Lambda\Lambda^T=\frac{\Gamma(\frac{d+1}{2})^2}{2\Gamma(\frac{d}{2}+1)^2}\E_\mu[\|X\|^2]\mathbf{I},
\end{equation}
$\mathbf{U},\mathbf{V}^T$ are unitary, and $\Gamma(\cdot)$ is gamma function. 
\end{theorem}

For the proof sketch, rather than finding the closed-form solution of the original sliced WGAN \eqref{eq_SlicedWGAN} directly, we solve its upper bounds with bounded-error and approximations to overcome the complicated integration and projection of \eqref{eq_SlicedWGAN}. First, by replacing general generator $G_\Theta (Z)$ with linear generator \eqref{eq_highdim}, \eqref{eq_SlicedWGAN} is upper-bounded by
\begin{equation} \label{eq_linearsliced}
 \min_\Theta \left(\int_{\omega\in\Omega}\inf_{\pi_\omega\in \Pi_\omega(\mu_\omega,\nu^\Theta_\omega)} \mathbb{E}_{\pi_\omega}[|\omega^T X - \omega^T\Theta Z|^q] d\omega \right)^{\frac{1}{q}}.
\end{equation}
Motivated by \cite{ReevesISIT17}, we replace the uniformly distributed $\omega\in\Omega$ in \eqref{eq_linearsliced} with Gaussian projections $\omega \in\Omega_G$, that is, $\omega$ is generated according to a zero-mean Gaussian distribution with covariance matrix $(1/d)\mathbf{I}_d$ and aim at solving
\begin{equation} \label{eq_GaussianManiProp}
 \min_\Theta \left(\int_{\omega\in\Omega_G}\inf_{\pi_\omega\in \Pi_\omega(\mu_\omega,\nu^\Theta_\omega)} \mathbb{E}_{\pi_\omega}[|\omega^T X - \omega^T\Theta Z|^q] d\omega \right)^{\frac{1}{q}},
\end{equation}
where we use notation $\int_{\omega\in\Omega_G} a(\omega) d\omega:=\int_{\omega\in\mathbb{R}^d} a(\omega) dF(\omega)$ to represent the expectation of function $a(\omega)$ over probability density function (PDF) $dF(\omega)$ of Gaussian random vectors. Using \cite[Proposition 1]{nadjahi2021fast}, it can be easily shown that the optimal $\Theta$s for \eqref{eq_linearsliced} and \eqref{eq_GaussianManiProp} are identical. Moreover, when $q=2$, not only the optimizers but also the values in \eqref{eq_linearsliced} and \eqref{eq_GaussianManiProp} become the same.
With proof sketch and comparisons with \cite{ReevesISIT17}\cite{nadjahi2021fast} given in Sec \ref{proofoptimalslicedWasserstein}, now \eqref{eq_GaussianManiProp} can be approximated with bounded error in the following Lemma with weak-dependent data $X$.
\begin{lemma}\label{ThmGaussianMani}
Let $\hat{X}_G \sim \mathcal{N}(0,\tilde{\sigma}^2_x)$ independent of Gaussian vector $\omega$ with finite
$
\tilde{\sigma}^2_x = \E_\mu[\|X\|^2]/d.
$
The gap of \eqref{eq_GaussianManiProp} to
\begin{equation} \label{eq_GaussianManiAppr}
\min_\Theta\left(\int_{\omega\in\Omega_G}\inf_{\pi\in\Pi(\mathcal{N}(0,\tilde{\sigma}^2_x), \nu^\Theta_\omega)} \mathbb{E}_{\pi}[|\hat{X}_G - \omega^T\Theta Z|^q] d\omega \right)^{1/q}
\end{equation}
is bounded by a function $f^\mu(d)=O(d^{-1/8})$, where the marginals of $\pi$ are $\mathcal{N}(0,\tilde{\sigma}^2_x)$ form $\hat{X}_G$ and $\nu^\Theta_\omega$ from $\omega^T\Theta Z$.
Solving \eqref{eq_GaussianManiAppr} with the linear generator \eqref{eq_highdim}, the optimal $\Theta$ equals to that of
\begin{equation} \label{eq_GaussianManiAppr20}
\min_\Theta \E_\omega [ |\tilde{\sigma}_x-\sqrt{\omega^T\Theta\Theta^T\omega}|^q ], q=1,2.
\end{equation}
Moreover, the optimal $\Theta$ when $q = 2$ is the minimizer to
\begin{equation} \label{eq_W2GaussianManiAppr2}
 \frac{\mathrm{Tr}(\Theta\Theta^T)}{d}+\frac{2\tilde{\sigma}_x }{\Gamma(1/2)} \int^\infty_0 z^{-1/2}\frac{\partial}{\partial z}\left|\mathbf{I}+\frac{2 z}{d} \Theta\Theta^T\right|^{-1/2}dz;
\end{equation}
where $\mathrm{Tr}(.)$ is the matrix trace.
 \end{lemma}
As summary to the current steps, the upper-bounds of \eqref{eq_SlicedWGAN} under $q=2$ are
\begin{align} \label{upperboundofslice}
&\min_\Theta\left(\int_{\omega\in\Omega}\inf_{\pi_\omega\in \Pi_\omega(\mu_\omega,\nu^\Theta_\omega)} \mathbb{E}_{\pi_\omega}[|\omega^T X - \omega^TG_\Theta (Z)|^2] d\omega \right)^{\frac{1}{2}}\notag\\
\le&\min_\Theta\left(\int_{\omega\in\Omega_G}\inf_{\pi_\omega\in \Pi_\omega(\mu_\omega,\nu^\Theta_\omega)} \mathbb{E}_{\pi_\omega}[|\omega^T X - \omega^T\Theta Z|^2] d\omega \right)^{\frac{1}{2}}\notag\\
\le&\min_\Theta\left(\int_{\omega\in\Omega_G}\inf_{\pi} \mathbb{E}_{\pi}[|\hat{X}_G - \omega^T\Theta Z|^2] d\omega \right)^{1/2}+O(d^{-1/8}).
\end{align}
where the optimal $\Theta$ in \eqref{upperboundofslice} can be solved by \eqref{eq_W2GaussianManiAppr2}. Then, let $r=d$, the final upper-bound in \eqref{upperboundofslice} becomes
\begin{equation}\label{ubvalueofslice}
\tilde{\sigma}_x\sqrt{1-\frac{d\Gamma(\frac{d+1}{2})^2}{2\Gamma(\frac{d}{2}+1)^2}}+O(d^{-1/8}).
\end{equation} 
\noindent Indeed one can show that \eqref{ubvalueofslice} tends to zero as $d\rightarrow \infty$, and thus linear generator in the second upper-bound of \eqref{upperboundofslice} is asymptotically optimal for \eqref{eq_SlicedWGAN}. Thus, in the sense of asymptotic optimality defined in Sec. \ref{sec:Problem}, even with the non-linear generator, the value of sliced WGAN under very high-dimensional non-Gaussian data cannot be smaller than that from the proposed linear generator \eqref{optimallinearparameter}. 

To solve the value \eqref{ubvalueofslice} of the upper bound \eqref{upperboundofslice}, we first apply singular value decomposition
(SVD) on parameters $\Theta=\mathbf{U}\Lambda\mathbf{V}^T$, and further simplify the objective using the Carlson-R function $R_{-a}(\mathbf{b}; \mathbf{z})$ \cite{carlson1977special}, where $\mathbf{b}, \mathbf{z}$ are $\R^d$ vectors. Let $\mathbf{S}=\Lambda\Lambda^T$. Now, the optimization problem becomes
\begin{equation}\label{R-form_optimal_prob}
\min_\mathbf{S}\frac{\mathrm{Tr}(\mathbf{S})}{d}-\tilde{\sigma}_x\sqrt{2d}\frac{\Gamma({\frac{d+1}{2})}}{\Gamma({\frac{d}{2}+1)}}R_{\frac{1}{2}}(\frac{1}{2}\mathbf{1};\mathrm{diag}(\mathbf{S})).
\end{equation}
Then, we prove that the objective of \eqref{R-form_optimal_prob}  is Schur-convex \cite{marshall11}, so the global minimum of \eqref{R-form_optimal_prob} is achieved when all the diagonal elements of $\mathbf{S}$ are equal. Thus, we find the optimal parameters \eqref{optimallinearparameter} of the linear generator, and obtain the upper bound value \eqref{ubvalueofslice}. The full proof of solving the upper bound is given in Sec. \ref{proofoptimalslicedWasserstein}.

\begin{remark}\label{remark:correlated}
From \cite{nadjahi2021fast}, under the weak dependence condition, the elements of the data vector $X$ can follow a Bernoulli shift, an autoregressive (AR) model, and a Gaussian process. For channel prediction, the AR model-based method is widely used by using past channel data without requiring current channel measurements via pilot \cite{wu2021channel}.
\end{remark}

\begin{remark}
As a sanity check, when $d=2$ and $q=2$, \eqref{eq_GaussianManiAppr} with linear generator can also be solved by elliptic integrals. The solution is the same as \eqref{optimallinearparameter}. The proof is given in appendix \ref{dimension2}.
\end{remark}
For the unprojected sliced WGAN \eqref{eq_SlicedWGANwithdata}, we show that the linear generator is also asymptotically optimal when $q=2$.
\begin{theorem}\label{thm:unprojectedsliced}
Under the assumption of Theorem \ref{optimalslicedWasserstein} with $\mu$ vanishing on hypersurfaces in $\R^d$, for the unprojected sliced WGAN \eqref{eq_SlicedWGANwithdata}, the linear generator \eqref{eq_highdim} with the same parameters as \eqref{optimallinearparameter} is asymptotically optimal when $q=2$.
\end{theorem}
\begin{proof}
We show that \eqref{eq_SlicedWGANwithdata} is indeed equivalent to \eqref{eq_SlicedWGAN} when $q=2$ by the following lemma.
\begin{lemma}\label{inner_discriminator_unproj_equal_proj}
For any $\omega\in\Omega$, the values of the inner-discriminator probelms of \eqref{eq_SlicedWGANwithdata} and \eqref{eq_SlicedWGAN} with the linear generator are the same when $q=2$, 
\begin{equation}\label{innerdiscriminatorequality}
\begin{aligned}
\inf_{\pi'\in \Pi'(\mu,\nu^\Theta_\omega)} \mathbb{E}_{\pi'}[|\omega^T X - \omega^T\Theta Z|^2]=\inf_{\pi_\omega\in \Pi_\omega(\mu_\omega,\nu^\Theta_\omega)} \mathbb{E}_{\pi_\omega}[|\omega^T X - \omega^T\Theta Z|^2].
\end{aligned}
\end{equation}
\end{lemma}
\noindent Though the optimal transport functions for the L.H.S. and R.H.S. are clearly different by definition, we prove that their corresponding values are the same. First, the L.H.S. in \eqref{innerdiscriminatorequality} belongs to the general multi-to-one-dimensional optimal transport problems with closed-form solutions identified in \cite{chiappori2017multi}. Then, we identify its analytical properties to prove that the optimal transport function of the L.H.S. of \eqref{innerdiscriminatorequality} is well-defined, and its value equals that of the R.H.S. The detailed proof of Lemma \ref{inner_discriminator_unproj_equal_proj} is in Sec. \ref{proof_inner_discriminator_unproj}. Finally, we can prove the desired result for \eqref{eq_SlicedWGANwithdata} following the steps to prove that for the original sliced WGAN \eqref{eq_SlicedWGAN} in Theorem \ref{optimalslicedWasserstein}.
\end{proof}

For application, our upcoming simulations in Sec. \ref{sec:simu} show that the linear generator \eqref{optimallinearparameter} can be a low-complexity generator with empirical data. For data without the weak dependence condition, the state-of-the-art WGAN solver can be obtained via gradient-flow of stochastic differential equation \cite{huang2023gans}\cite{huang2024generative}, and our approach can be a low-complexity initialization and has the potential for faster convergence.

\section{The Proofs}\label{sec:proof}
Here we give proofs for our main theorems.

\subsection{Proof for Theorem \ref{thm:theta's ReLU}} \label{proofReLU} 
In the following, we first solve the first sub-problem with $\theta_2 \in \R_+$ in \eqref{eqW2sub}. With $d=1$, we recall the following result for the inner-discriminator problem \eqref{P} of \eqref{eq_transportGANE}, by taking $\nu^\Theta := P_{G_\Theta(z)}$, the measure generated by the generator NN (with parameter
$\Theta:=(\theta_1,\theta_2)$ in \eqref{G}). Let $F_\mu$ and $F_{\nu^\Theta}$ denote the CDFs of the measures $\mu$ and $\nu^\Theta$ on $\R$.

\begin{lemma}[\cite{AP03}, Theorem 5.1] \label{prop:1-d result}
Define $t^\Theta:\R\to \R\cup\{\infty\}$ by
\begin{equation}\label{quantile}
t^\Theta(x):= \sup\{y\in\R: F_{\nu^\Theta}(y)\le F_\mu(x)\}.
\end{equation}
For $q=1,2$, if $\mu$ has no atom ($\mu$ is a continuous distribution)
\begin{align}
&\inf_{\pi\in \Pi(\mu,\nu^\Theta)} \int_{\R\times\R} |x-y|^q d\pi(x,y) = \int_\R |x-t^\Theta(x)|^q d\mu(x), \label{Wd_2}
\end{align}
\end{lemma}

For the inner discriminator problem \eqref{P} with $q=2, d=1$, $\E_\mu [P(X,\theta_1,\theta_2)]$ for given $(\mu, \theta_1, \theta_2)$ equals to \eqref{Wd_2}, where $t^\Theta(x)$ is defined as in \eqref{quantile}.

 Now we need to find a closed-form $t^\Theta(x)$ to continue. From \eqref{G}, let $\Psi$ denote the CDF of $h(Z)$. As $\Psi$ is neither continuous nor strictly increasing, $t^\Theta(x)=\theta_1+\theta_2\Psi^{-1}(F_\mu(x))$ does not hold in general, and we cannot directly apply the proof of Proposition \ref{prop:theta's} in Appendix \ref{proofSimoid}. Assume $\theta_2 \in \R_+$, our target sub-problem becomes \begin{equation}\label{GAN-sub}
\min_{\theta_1\in\R,\theta_2\ge 0}  \int_\R |x-t^\Theta(x)|^2 d\mu(x).
\end{equation}
Under the condition $\theta_2\ge 0$, observe that
\begin{equation}\label{gett^theta}
\begin{aligned}
\Psi\bigg(\frac{t^\Theta(x)-\theta_1}{\theta_2}\bigg) =   \P\bigg(h(Z) \le \frac{t^\Theta(x)-\theta_1}{\theta_2}\bigg)
= \P\left(G_\Theta(Z) \le t^\Theta(x)\right) = F_{\nu^\Theta}\left(t^\Theta(x)\right).
\end{aligned}
\end{equation}
Only for $F_\mu(x) > 1/2$, \eqref{gett^theta} equals to $F_\mu(x)$ from \eqref{quantile}. From \eqref{cdf ReLU},
\begin{equation}\label{quantile' ReLU}
t^\Theta(x) =
\begin{cases}
\theta_1,\quad &\hbox{if}\ F_\mu(x)\le 1/2,\\
\theta_1 + \theta_2 \Phi^{-1}(F_\mu(x)),\quad &\hbox{if}\ F_\mu(x)> 1/2.
\end{cases}
\end{equation}

With closed-form representation \eqref{Wd_2}\eqref{quantile' ReLU} for the inner problem \eqref{P}, WGAN \eqref{eq_transportGANE} with $q=2,d=1$ can be simplified to be the following stochastic minimization problem
\begin{equation}\label{GAN' ReLU}
\min_{\theta_1 \in\R, \theta_2 \in\R_+}  \E_\mu\left[\left|X-t^\Theta(X)\right|^2 \right].
\end{equation} 
Together with \eqref{G}, the sub-problem \eqref{GAN' ReLU} becomes the constrained optimization
problem
\begin{equation}\label{constrained ReLU}
\begin{split}
&\min_{\theta_1,\theta_2\in\R}  J(\theta_1,\theta_2) \\
:=&  \int_{\{F_\mu(x)\le 1/2\}} \Big(\theta_1 -x\Big)^2 F_\mu'(x) dx+\int_{\{F_\mu(x)> 1/2\}} \Big(\theta_1 + \theta_2 \Phi^{-1}(F_\mu(x))-x\Big)^2 F_\mu'(x) dx\\
&\hbox{subject to}\quad g(\theta_1,\theta_2) := -\theta_2\le 0.
\end{split}
\end{equation}
We will show that even with a jump, the KKT condition of \eqref{constrained ReLU} can be simplified as 
\begin{equation}\label{KKT1''}
\begin{split}
\theta_1 &= \E[X] -  \frac{\theta_2}{\sqrt{2\pi}},\\
\lambda/2 &= \frac{\theta_2}{2} \left(1-\frac{1}{\pi}\right) - \text{Cov}(X,\Phi^{-1}(F_\mu(X)) 1_{\{F_\mu(X)> 1/2\}}).
\end{split}
\end{equation}
where $\lambda\ge 0$ is the Lagrange multiplier. 

The corresponding first-order KKT conditions for \eqref{constrained ReLU} are 
\begin{align}
\nabla J(\theta_1,\theta_2) +\lambda \nabla g(\theta_1,\theta_2) &= 0,\label{KKT1}\\
\lambda g(\theta_1,\theta_2) &=0.\label{KKT2}
\end{align}
By direct calculation, \eqref{KKT1} becomes
\begin{equation}
\begin{aligned}
\int_{\{F_\mu(x)\le 1/2\}}  \Big(\theta_1 -x\Big)  F_\mu'(x) dx+\int_{\{F_\mu(x)> 1/2\}}  \Big(\theta_1 + \theta_2 \Phi^{-1}(F_\mu(x))-x\Big)  F_\mu'(x) dx =0,\label{11}
\end{aligned}
\end{equation}
\begin{equation}
\begin{aligned}
&\int_{\{F_\mu(x)> 1/2\}}  \Big(\theta_1 + \theta_2 \Phi^{-1}(F_\mu(x))-x\Big) \Phi^{-1}(F_\mu(x))F_\mu'(x)dx=\frac{\lambda}{2}.\label{22}
\end{aligned}
\end{equation}
Recall that $X$ is a random variable with CDF $F_\mu$. Then, \eqref{11} and \eqref{22} can be written as
\begin{equation}
\begin{aligned}
&\theta_1 + \theta_2 \E\left[\Phi^{-1}(F_\mu(X)) 1_{\{F_\mu(X)> 1/2\}} \right] -\E[X]=0,\label{11'}
\end{aligned}
\end{equation}
and
\begin{equation}
\begin{aligned}
&\theta_1  \E\left[\Phi^{-1}(F_\mu(X)) 1_{\{F_\mu(X)> 1/2\}}\right]
+ \theta_2  \E\left[(\Phi^{-1}(F_\mu(X)))^2 1_{\{F_\mu(X)> 1/2\}}\right]\\
&-  \E\left[X \Phi^{-1}(F_\mu(X)) 1_{\{F_\mu(X)> 1/2\}}\right] -\lambda/2=0,\label{22'}
\end{aligned}
\end{equation}
respectively. Plugging 
\[\theta_1 = \E[X] -  \theta_2 \E\left[\Phi^{-1}(F_\mu(X)) 1_{\{F_\mu(X)> 1/2\}}\right],\]
(from \eqref{11'}), into \eqref{22'}, we obtain
\begin{align}
&\E[X] \E\left[\Phi^{-1}(F_\mu(X)) 1_{\{F_\mu(X)> 1/2\}} \right]
+\theta_2 \text{Var}\left(\Phi^{-1}(F_\mu(X))1_{\{F_\mu(X)> 1/2\}}\right)\notag\\
&- \E\left[X \Phi^{-1}(F_\mu(X)) 1_{\{F_\mu(X)> 1/2\}} \right]-\lambda/2=0. \label{22''} 
\end{align}
Hence, a solution $(\theta_1,\theta_2,\lambda)$ to \eqref{KKT1} must equivalently satisfy \eqref{11'} and \eqref{22''}.
Note that $F_\mu(X)\sim\operatorname{Uniform}[0,1]$ from \cite[Lemma
1]{shayevitz2011optimal}, so that the CDF of $\Phi^{-1}(F_\mu(X))$ is simply $\Phi$. That is, 
\begin{equation}\label{identical}
\Phi^{-1}(F_\mu(X))\overset{(d)}{=}Z\sim\mathcal{N}(0,1),
\end{equation} 
where $\overset{(d)}{=}$ denotes equality in distribution. It follows that
\begin{align}
&\E\left[\Phi^{-1}(F_\mu(X)) 1_{\{F_\mu(X)> 1/2\}} \right] 
= \E[Z 1_{\{\Phi(Z)>1/2\}}] = \E[Z 1_{\{ Z>0\}}] = \frac{1}{\sqrt{2\pi}},\label{E calc}\\
&\E\left[(\Phi^{-1}(F_\mu(X)))^2 1_{\{F_\mu(X)> 1/2\}} \right] 
=  \E[Z^2 1_{\{\Phi(Z)>1/2\}}] = \E[Z^2 1_{\{ Z>0\}}] =  \frac{1}{2}. \label{E^2 calc}
\end{align}
In view of this, \eqref{11'} and \eqref{22''} reduce to \eqref{KKT1''}.

\begin{lemma}\label{prop:theta's ReLU}
Suppose that $\mu$ has no atom. Then, \eqref{constrained ReLU} has a unique minimizer $(\theta_1^*,\theta_2^*)\in\R\times\R_+$ given by
\begin{equation}\label{theta's ReLU}
\begin{split}
\theta^*_1 &=  \E[X] -  \frac{\sqrt{2\pi}}{\pi-1} \Cov\big(X, \Phi^{-1}(F_\mu(X)) 1_{\{F_\mu(X)> 1/2\}}  \big),\\
\theta^*_2 &=   \frac{2\pi}{\pi-1} \Cov\left(X, \Phi^{-1}(F_\mu(X)) 1_{\{F_\mu(X)> 1/2\}} \right)\ge 0.
\end{split}
\end{equation}
\end{lemma}

\begin{proof}
We first show that the covariance term in \eqref{KKT1''} is non-negative. As $\Phi^{-1}$ is strictly increasing and $F_\mu$ is nondecreasing, the map $x\mapsto \Phi^{-1}(F_\mu(x)) 1_{\{F_\mu(x)> 1/2\}}$ is nondecreasing. This readily implies
\begin{equation}\label{eq_positiveCovReLU}
\Cov\left(X, \Phi^{-1}(F_\mu(X))1_{\{F_\mu(X)> 1/2\}}\right)\ge0
\end{equation}
in \eqref{KKT1''} since 
\begin{equation}
\begin{aligned}
&(x-x')\big(\Phi^{-1}(F_\mu(x))1_{\{F_\mu(x)> 1/2\}} - \Phi^{-1}(F_\mu(x'))1_{\{F_\mu(x')> 1/2\}}\big)\ge0 \notag
\end{aligned}
\end{equation}
for all $x, x'\in\R$. Specifically, by taking an independent but same distribution copy $X'$ of $X$,
\begin{equation}\label{Cov nonn}
\begin{aligned}
0&\le \E[(X-X')\cdot \left(\Phi^{-1}(F_\mu(X))1_{\{F_\mu(X)> 1/2\}} - \Phi^{-1}(F_\mu(X'))1_{\{F_\mu(X')> 1/2\}}\right)] \\
& = 2 \Cov\left(X, \Phi^{-1}(F_\mu(X))1_{\{F_\mu(X)> 1/2\}}\right).
\end{aligned}
\end{equation}
Specifically, we deal with the two cases where $\Cov\left(X, \Phi^{-1}(F_\mu(X)) 1_{\{F_\mu(X)> 1/2\}}  \right)$ is strictly positive and zero separately. 

\noindent
{\bf Case I:} $\Cov\left(X, \Phi^{-1}(F_\mu(X)) 1_{\{F_\mu(X)> 1/2\}}  \right)> 0$. If $\theta_2 =0$, \eqref{KKT1''} entails 

\noindent$\lambda= -2\Cov\left(X, \Phi^{-1}(F_\mu(X)) 1_{\{F_\mu(X)> 1/2\}}  \right)<0$, which violates the requirement $\lambda\ge 0$ in \eqref{KKT1}-\eqref{KKT2}. If $\lambda=0$, solving \eqref{KKT1''} yields \eqref{theta's ReLU}.

That is, the KKT condition gives a unique candidate optimizer $(\theta_1^*,\theta_2^*,\lambda^*)$, with $(\theta_1^*,\theta_2^*)$ as in \eqref{theta's ReLU} and $\lambda^*=0$. To check the corresponding second-order condition, let $H_J$ and $H_g$ denote the Hessian matrices of $J(\theta_1,\theta_2)$ and $g(\theta_1,\theta_2)$, respectively. For $H_J$,
\[\frac{\partial^2 J(\theta_1,\theta_2)}{{\partial\theta_1}^2}=2,
\]
\[\frac{\partial^2 J(\theta_1,\theta_2)}{\partial\theta_1\partial\theta_2}=2\E\left[\Phi^{-1}(F_\mu(X)) 1_{\{F_\mu(X)> 1/2\}} \right],
\]from \eqref{E calc} and 
\[
\frac{\partial^2 J(\theta_1,\theta_2)}{{\partial\theta_2}^2}=2\E\left[(\Phi^{-1}(F_\mu(X)))^2 1_{\{F_\mu(X)> 1/2\}}\right], 
\]
from \eqref{E^2 calc}.
Clearly, $H_g = O_{2\times 2}$, which implies
\begin{equation}\label{Hessian ReLU}
\begin{aligned}
H_J + \lambda H_g
=2\begin{bmatrix}
1 &  \frac{1}{\sqrt{2\pi}}\\
 \frac{1}{\sqrt{2\pi}}& \frac12
 \end{bmatrix},
\end{aligned}
\end{equation}
which is positive definite. Hence, we conclude that $(\theta_1^*,\theta_2^*)$ in \eqref{theta's ReLU} is the unique minimizer of \eqref{constrained ReLU}. \\
\noindent {\bf Case II:} $\Cov\left(X, \Phi^{-1}(F_\mu(X)) 1_{\{F_\mu(X)> 1/2\}}  \right)= 0$. As \eqref{KKT1''} entails $\lambda = \theta_2 (1-\frac{1}{\pi})$, we have $\lambda=\theta_2 =0$ (by recalling complementary slackness that either $\lambda =0$ or $\theta_2 =0$). That is, the KKT condition gives a unique candidate optimizer $(\theta_1^*,\theta_2^*,\lambda^*) = (\E[X],0,0)$. Since $H_J+\lambda H_g$ is positive definite, we conclude that $(\theta_1^*,\theta_2^*) = (\E[X],0)$ is the unique minimizer of \eqref{constrained ReLU}.

\noindent  As a remark, another way to prove is to show $H_J$ as \eqref{Hessian  ReLU} is positive definite in both cases first, and then $J(\theta_1, \theta_2)$ is convex. Since both proofs require checking the Hessian matrix, there is no simpler way to prove.
\end{proof}

\begin{lemma}\label{coro:ReLU}
Suppose that $\mu$ has no atom and $h(x)=\max\{x,0\}$ for all $x\in\R$.
The sub-problem \eqref{GAN-sub} has a unique minimizer $\Theta^* = (\theta_1^*,\theta_2^*)\in\R\times\R_+$, defined as in \eqref{theta's ReLU}. Moreover, the minimum value is
\begin{equation}\label{min value ReLU}
\begin{split}
&\int_\R |x-t^{\Theta^*}(x)|^2 d\mu(x)
= \Var(X) - \frac{2\pi}{\pi-1} \Cov\left(X,\Phi^{-1}(F_\mu(X))1_{\{F_\mu(X)>1/2\}}\right)^2.
\end{split}
\end{equation}
\end{lemma}

\begin{proof}
As \eqref{GAN-sub} is equivalent to \eqref{constrained ReLU}, Lemma~\ref{prop:theta's ReLU} directly shows that $\Theta^* = (\theta_1^*,\theta_2^*)\in\R\times\R_+$ as in \eqref{theta's ReLU} is the unique minimizer of \eqref{GAN-sub}. Thanks to \eqref{quantile' ReLU} and \eqref{theta's ReLU},
\begin{align}
&\int_\R|x-t^{\Theta^*}(x)|^2 d\mu(x) \notag\\
=& \int_{\{F_\mu(x)\le 1/2\}} \left(x-\E[X] + \frac{\theta^*_2}{\sqrt{2\pi}} \right)^2 d\mu(x)
+ \int_{\{F_\mu(x)> 1/2\}} \left(x-\E[X] -\frac{\theta^*_2}{\sqrt{2\pi}}  -\theta^*_2\Psi^{-1}(F_\mu(x))\right)^2 d\mu(x)\notag\\
=& \int_{\{F_\mu(x)\le 1/2\}} (x-\E[X])^2 d\mu(x)
+\sqrt{\frac{2}{\pi}}\theta^*_2 \int_{\{F_\mu(x)\le 1/2\}} (x-\E[X])d\mu(x)
+ \frac{(\theta^*_2)^2}{2\pi} \int_\R 1_{\{F_\mu(x)\le 1/2\}} d\mu(x)\notag\\
+& \int_{\{F_\mu(x)> 1/2\}} (x-\E[X])^2 d\mu(x) 
-2\theta^*_2 \int_{\{F_\mu(x)>1/2\}} (x-\E[X]) \Big(\Psi^{-1}(F_\mu(x))-\frac{1}{\sqrt{2\pi}}\Big)d\mu(x)\notag\\
+& (\theta^*_2)^2 \int_{\{F_\mu(x)>1/2\}} \Big(\Psi^{-1}(F_\mu(x))-\frac{1}{\sqrt{2\pi}}\Big)^2 d\mu(x).\label{minv1}
\end{align}
Observe that
\begin{align}
& \int_{\{F_\mu(x)>1/2\}} (x-\E[X]) \Big(\Psi^{-1}(F_\mu(x))-\frac{1}{\sqrt{2\pi}}\Big)d\mu(x)\notag\\
=& \int_\R (x-\E[X]) \Big(\Psi^{-1}(F_\mu(x)) 1_{\{F_\mu(x)>1/2\}}
-\frac{1}{\sqrt{2\pi}} 1_{\{F_\mu(x)>1/2\}}\Big)d\mu(x)\notag\\
=&  \int_\R (x-\E[X]) \Big(\Psi^{-1}(F_\mu(x)) 1_{\{F_\mu(x)>1/2\}}-\frac{1}{\sqrt{2\pi}} \Big)d\mu(x)
+ \frac{1}{\sqrt{2\pi}}\int_\R(x-\E[X]) 1_{\{F_\mu(x)\le 1/2\}} d\mu(x)\notag\\
=& \Cov\left(X, \Psi^{-1}(F_\mu(x)) 1_{\{F_\mu(x)>1/2\}}\right) 
+ \frac{1}{\sqrt{2\pi}}\int_\R(x-\E[X]) 1_{\{F_\mu(x)\le 1/2\}} d\mu(x),\label{=Cov}
\end{align}
where the last equality follows from \eqref{E calc}. Similarly,
\begin{align}
& \int_{\{F_\mu(x)>1/2\}} \Big(\Psi^{-1}(F_\mu(x))-\frac{1}{\sqrt{2\pi}}\Big)^2 d\mu(x)\notag\\
=& \int_\R  \Big(\Psi^{-1}(F_\mu(x)) 1_{\{F_\mu(x)>1/2\}}-\frac{1}{\sqrt{2\pi}} 1_{\{F_\mu(x)>1/2\}}\Big)^2 d\mu(x)\notag\\
=&  \int_\R \Big(\Psi^{-1}(F_\mu(x)) 1_{\{F_\mu(x)>1/2\}}-\frac{1}{\sqrt{2\pi}} \Big)^2 d\mu(x)
- \frac{1}{2\pi}\int_\R  1_{\{F_\mu(x)\le 1/2\}} d\mu(x)\notag\\
=& \Var\left(\Psi^{-1}(F_\mu(x)) 1_{\{F_\mu(x)>1/2\}}\right) - \frac{1}{2\pi}\int_\R 1_{\{F_\mu(x)\le 1/2\}} d\mu(x)\notag\\
=& \frac{\pi-1}{2\pi} - \frac{1}{2\pi}\int_\R 1_{\{F_\mu(x)\le 1/2\}} d\mu(x),\label{=Var}
\end{align}
where the last two equalities follow from \eqref{E calc} and \eqref{E^2 calc}. Thanks to \eqref{=Cov} and \eqref{=Var}, we conclude from \eqref{minv1} that
\begin{align*}
&\int_\R |x-t^{\Theta^*}(x)|^2 d\mu(x) \\
=& \Var(X) - 2\theta^*_2 \Cov\left(X, \Psi^{-1}(F_\mu(X)) 1_{\{F_\mu(x)> 1/2\}} \right)+ (\theta^*_2)^2  \frac{\pi-1}{2\pi} \\
=&\Var(X) - \frac{2\pi}{\pi-1} \Cov\left(X, \Psi^{-1}(F_\mu(X)) 1_{\{F_\mu(x)> 1/2\}} \right)^2,
\end{align*}
where the last equality is due to $\theta^*_2$ in \eqref{theta's ReLU}.
\end{proof}

Next, we go to $\theta_2\le 0$. The target sub-problem is 
\begin{equation}\label{GAN-sub'}
\min_{\theta_1\in\R,\theta_2\le 0}  \int_\R |x-t^\Theta(x)|^2 d\mu(x).
\end{equation}

\begin{lemma}\label{ReLUthetanegative}
Suppose that $\mu$ has no atom and $h(x)=\max\{x,0\}$ for all $x\in\R$. The sub-problem \eqref{GAN-sub'} has a unique minimizer $\Theta^* = (\theta_1^*,\theta_2^*)$, defined by
\begin{equation}\label{theta's ReLU'}
\begin{split}
\theta^*_1 &=  \E[X] + \frac{\sqrt{2\pi}}{\pi-1} \Cov\big(X, \Phi^{-1}(F_\mu(X)) 1_{\{F_\mu(X)\le 1/2\}}  \big),\\
\theta^*_2 &=   -\frac{2\pi}{\pi-1} \Cov\left(X, \Phi^{-1}(F_\mu(X)) 1_{\{F_\mu(X)\le1/2\}} \right).
\end{split}
\end{equation}
Moreover, the minimum value is
\begin{equation}\label{min value ReLU'}
\begin{aligned}
&\int_\R |x-t^{\Theta^*}(x)|^2 d\mu(x)
= \Var(X) - \frac{2\pi}{\pi-1} \Cov\left(X,\Phi^{-1}(F_\mu(X)) 1_{\{F_\mu(X)\le1/2\}}\right)^2.
\end{aligned}
\end{equation}
\end{lemma}
\noindent The proof for the case $\theta_2\le0$ is similar to the case $\theta_2\ge0$, and is given in Appendix \ref{prooflReLUnegative}
Finally, we can get our theorem. If \eqref{Cov>0 ReLU} holds, the value in \eqref{min value ReLU} is smaller than or equal to that in \eqref{min value ReLU'}. Hence, the minimizer $(\theta_1^*,\theta_2^*)$ should be the one associated with the value \eqref{min value ReLU}, which is given by \eqref{theta's ReLU}. On the other hand, if \eqref{Cov>0 ReLU} fails, the value in \eqref{min value ReLU'} is smaller than that in \eqref{min value ReLU}. The minimizer $(\theta_1^*,\theta_2^*)$ should then be the one associated with the value \eqref{min value ReLU'}, which is given by \eqref{theta's ReLU'}.

\subsection{Proof for Theorem \ref{optimalslicedWasserstein} } \label{proofoptimalslicedWasserstein}
In the beginning, here we provide the proof sketch of Lemma 1 and leave the details to Appendix \ref{ProofGaussianManifull}. First, we assume $X$ is zero mean, since equivalently one can modify $G_\Theta(Z)$ in \eqref{eq_SlicedWGAN} with a bias $G_\Theta(Z)+\E[X]$ for non-zero mean $X$. With $q=2$, Wasserstein distance (as RHS of \eqref{P}) of order 2 between $\omega^T X$ with $\hat{X}_G$, averaged over $\omega$, is bounded from \cite{ReevesISIT17}. We modify this result to make it valid for $q=1$, and then get desired bounded gap $f^\mu(d)$ between \eqref{eq_GaussianManiProp}\eqref{eq_GaussianManiAppr} for $q=1,2$ as
\begin{align} \label{eq_GaussianManiBound}
f^\mu(d):= (C' \E_\mu[\|X\|^2](d^{-5/4} + d^{-7/5}))^{1/2}
\end{align}
where $C'$ is a constant. With linear generator \eqref{eq_highdim}, not only $\omega^T\Theta Z$ is Gaussian as $\hat{X}_G$ given $\omega$, on the contrary to \cite[Theorem 1]{nadjahi2021fast}, but also we prevent gap $f^\mu(d)$ depending on $\Theta$.

Next \eqref{eq_GaussianManiAppr20} is from modifying Corollary \ref{prop:theta's Linear} and \ref{prop:W1d1} to \eqref{eq_GaussianManiAppr}. Finally, solving \eqref{eq_GaussianManiAppr20} with $q=2$ equals to solving
\begin{equation} \label{eq_W2GaussianManiAppr4}
 \tilde{\sigma}^2_x-\max_\Theta \E_w [ 2\tilde{\sigma}_x \sqrt{\omega^T\Theta\Theta^T\omega}-\omega^T\Theta\Theta^T\omega].
\end{equation}
 For the first term, for Gaussian $w \in \Omega_G$, let its quadratic form $U_w:=\omega^T\Theta\Theta^T\omega$ and $U_w>0$ almost surely. Then, $\E_w [\sqrt{\omega^T\Theta\Theta^T\omega}]=\E_{U_w}[U_wU_w^{-1/2}]$ and we can prove that
\begin{align} \label{eqSqrtQuadratic}
\E_w [\sqrt{\omega^T\Theta\Theta^T\omega}]= \frac{-1}{\Gamma(1/2)} \int^\infty_0 z^{-1/2}\frac{\partial}{\partial z}M_{U_w}(-z)dz
\end{align}
where $M_{U_w}(z):=\E[e^{zU_w}]$ is the moment generating function of $U_w$. For the second term, it is easy to see that
\begin{equation}\label{eq_W2GaussianManiAppr5}
\E_w [\omega^T\Theta\Theta^T\omega]=\mathrm{Tr}(\E_w [\omega\omega^T]\Theta\Theta^T)=\mathrm{Tr}(\Theta\Theta^T)/d.
\end{equation}
Combining these results reaches \eqref{eq_W2GaussianManiAppr2}.

Continuing the proof for Theorem \ref{optimalslicedWasserstein}, we first aim to solve \eqref{eq_W2GaussianManiAppr2}. Though the derivative with respect to $z$ in \eqref{eq_W2GaussianManiAppr2} can be easily solved by matrix calculus \cite{hjorungnes2011complex}, direct integration over $z$ is still too complicated to turn it into the closed-form \eqref{R-form_optimal_prob}. To solve this, applying SVD on the parameters of linear generator $\Theta=\mathbf{U}\Lambda \mathbf{V}^T\in\R^{d\times r}$, then $\Theta\Theta^T$ can be eigen-decomposed as \[
\Theta\Theta^T =\mathbf{U}\Lambda\Lambda^T\mathbf{U}^T=\mathbf{U}\mathbf{S}\mathbf{U}^T,
\] where $\mathbf{U}\in\R^{d\times d}$ is eigenvectors of $\Theta\Theta^T$, $\mathbf{V}\in\R^{r\times r}$ is eigenvectors of $\Theta^T\Theta$, $\Lambda\in\R^{d\times r}$ and $\mathbf{S}$ is a diagonal matrix. Now diagonal $\mathbf{S}=\Lambda\Lambda^T\in\R^{d\times d}$ has each diagonal element $S_{ii}\ge0$.
Substituting $\Theta\Theta^T=\mathbf{U}\mathbf{S}\mathbf{U}^T$ into \eqref{eq_W2GaussianManiAppr2}, it becomes
\begin{align}
& \frac{\sum_{i=1}^d S_{ii}}{d}+\frac{2\tilde{\sigma}_x }{\Gamma(1/2)} \int^\infty_0 z^{-1/2}\frac{\partial}{\partial z}(\prod_{i=1}^d (1+\frac{2 z}{d} S_{ii})^{-1/2})dz \label{eq_W2GaussianSVD} \\
 = &\frac{\sum_{i=1}^d S_{ii}}{d}-\frac{2\tilde{\sigma}_x }{\Gamma(1/2)} \int^\infty_0 \frac{1}{2}z^{-1/2}(\prod_{i=1}^d (1+\frac{2 z}{d} S_{ii})^{-3/2}) (\sum_{i=1}^d \frac{2}{d}S_{ii}\prod_{j=1, j\neq i}^d (1+\frac{2 z}{d} S_{jj}))dz. \label{eq_W2GaussianDerivative}
\end{align}

To solve the optimal elements of the diagonal matrix $\mathbf{S}$ in \eqref{eq_W2GaussianDerivative}, we focus on simplifying the integral part of it using the Carlson-R function and turn it into the objectives of \eqref{R-form_optimal_prob}. Details are as follows, in \eqref{eq_W2GaussianDerivative},
\begin{equation}{}
\begin{aligned}
&-\frac{2\tilde{\sigma}_x }{\Gamma(1/2)} \int^\infty_0 \frac{1}{2}z^{-1/2}(\prod_{i=1}^d (1+\frac{2 z}{d} S_{ii})^{-3/2}) (\sum_{i=1}^d \frac{2}{d}S_{ii}\prod_{j=1, j\neq i}^d (1+\frac{2 z}{d} S_{jj}))dz\\
=&-\frac{2\tilde{\sigma}_x }{\Gamma(1/2)} \sum_{i=1}^d\int^\infty_0 \frac{1}{d}z^{-1/2}S_{ii}(1+\frac{2 z}{d} S_{ii})^{-3/2} \prod_{j=1, j\neq i}^d (1+\frac{2 z}{d} S_{jj})^{-1/2}dz.\\
\overset{(a)}{=}&-\frac{2\tilde{\sigma}_x }{\Gamma(1/2)} \sum_{i=1}^d\frac{1}{d}B(\frac{1}{2},\frac{d+1}{2})S_{ii}R_{-\frac{1}{2}}\left(\mathbf{b}^{(i)};\mathrm{diag}\left(\frac{2}{d}\mathbf{S}\right)\right)\\
\end{aligned}
\end{equation}
The step (a) follows \cite[(8.3-1)]{carlson1977special} where $B(\cdot, \cdot)$ is beta function, $\mathbf{b}^{(i)}\in\mathbb{R}^d$ with elements $b_i^{(i)}= 3/2$ and $b_j^{(i)}= 1/2$ for $j\ne i$, and $R_{-a}(\mathbf{b}; \mathbf{z})$ is the Carlson-R function
\[
\begin{aligned}
R_{-a}(\mathbf{b}; \mathbf{z})&:=\frac{1}{B(a, \sum{b_i})}\int_0^\infty t^{\sum{b_i}-a}\prod(t+z_i)^{-b_i}dt=\frac{1}{B(a, \sum{b_i})}\int_0^\infty t^{a-1}\prod(1+tz_i)^{-b_i}dt.
\end{aligned}
\]

We use some properties of the Carlson-R function to simplify (a) further
\begin{equation}\label{CarlsonRderives}
\begin{aligned}
&-\frac{2\tilde{\sigma}_x }{\Gamma(1/2)} \sum_{i=1}^d\frac{1}{d}B(\frac{1}{2},\frac{d+1}{2})S_{ii}R_{-\frac{1}{2}}\left(\mathbf{b}^{(i)};\mathrm{diag}\left(\frac{2}{d}\mathbf{S}\right)\right)\\
\overset{(b)}{=}&-\frac{2\tilde{\sigma}_x }{\Gamma(1/2)} \sum_{i=1}^d\frac{1}{\sqrt{2d}}B(\frac{1}{2},\frac{d+1}{2})S_{ii}R_{-\frac{1}{2}}\left(\mathbf{b}^{(i)};\mathrm{diag}\left(\mathbf{S}\right)\right)\\
\overset{(c)}{=}&-\frac{2\tilde{\sigma}_x }{\Gamma(1/2)} \sqrt{2d}B(\frac{1}{2},\frac{d+1}{2})\frac{1}{2}R_{\frac{1}{2}}\left(\frac{1}{2}\mathbf{1};\mathrm{diag}\left(\mathbf{S}\right)\right)\\
=&-\tilde{\sigma}_x\sqrt{2d}\frac{\Gamma({\frac{d+1}{2})}}{\Gamma({\frac{d}{2}+1)}}R_{\frac{1}{2}}\left(\frac{1}{2}\mathbf{1};\mathrm{diag}\left(\mathbf{S}\right)\right);
\end{aligned}
\end{equation}
the step (b) follows \cite[(5.9-3)]{carlson1977special} that 
\[R_{-\frac{1}{2}}\left(\mathbf{b}^{(i)};\mathrm{diag}\left(\frac{2}{d}\mathbf{S}\right)\right)=(\frac{2}{d})^{-1/2}R_{-\frac{1}{2}}\left(\mathbf{b}^{(i)};\mathrm{diag}\left(\mathbf{S}\right)\right),
\]
and the step (c) comes from Euler's homogeneity relation \cite[(5.9-2)]{carlson1977special}
\[
\sum^{d}_{i=1}S_{ii}\frac{\partial}{\partial S_{ii}}R_{\frac{1}{2}}\left(\frac{1}{2}\mathbf{1};\mathrm{diag}\left(\mathbf{S}\right)\right) = \frac{1}{2}R_{\frac{1}{2}}\left(\frac{1}{2}\mathbf{1};\mathrm{diag}\left(\mathbf{S}\right)\right),
\] and the derivative of R function \cite[(5.9-9)]{carlson1977special}\[
\frac{\partial}{\partial S_{ii}}R_{\frac{1}{2}}\left(\frac{1}{2}\mathbf{1};\mathrm{diag}\left(\mathbf{S}\right)\right)=\frac{1}{2d}R_{-\frac{1}{2}}\left(\mathbf{b}^{(i)};\mathrm{diag}\left(\mathbf{S}\right)\right).
\] 
The last equality of \eqref{CarlsonRderives} comes from 
\[B(\frac{1}{2},\frac{d+1}{2})=\Gamma(\frac{1}{2})\Gamma(\frac{d+1}{2})/\Gamma(\frac{d}{2}+1),
\]
and thus, \eqref{eq_W2GaussianDerivative} becomes the objective of \eqref{R-form_optimal_prob}.

Next, with the aid of Dirichlet average \cite{carlson1977special}, we prove that the objective of \eqref{R-form_optimal_prob} is Schur-convex \cite{marshall11}, and it achieves the global minimum when all the diagonal elements of $\mathbf{S}$ are the same as $s$. 
\begin{lemma}\label{schurconvexofR}
The objective of the equivalent problem \eqref{R-form_optimal_prob} of \eqref{eq_W2GaussianManiAppr2} is symmetric and convex over $\mathbf{S}$, and such a Schur-convex objective \eqref{R-form_optimal_prob} has a global minimum when $\mathbf{S}=s\mathbf{I}$.
\end{lemma}
\begin{proof}
Carlson R function is a special case of Dirichlet average\cite[(5.9-1)]{carlson1977special}
\begin{equation} \label{Dirichletaverage}
R_{\frac{1}{2}}(\frac{1}{2}\mathbf{1};\mathrm{diag}(\mathbf{S}))=\int_{\mathbb{R}^{d-1}}(\mathbf{u}\cdot\mathrm{diag}(\mathbf{S}))^{1/2}d\mu_{\frac{1}{2}\mathbf{1}}(\mathbf{u}),
\end{equation}
where 
\[
d\mu_{\frac{1}{2}\mathbf{1}}(\mathbf{u}):=\frac{\Gamma(d/2)}{\Gamma(1/2)^d}u_1^{-1/2}u_2^{-1/2}\cdots u_d^{-1/2}du_1\cdots du_{d-1},
\]
and \[
\sum_{i=1}^d u_i=1,\ u_i\ge0.
\]
Note that \eqref{Dirichletaverage} is symmetric over $\mathbf{S}$ obviously, and thus the same property applies to \eqref{R-form_optimal_prob}. Since the Hessian matrix of $-(\mathbf{u}\cdot\mathrm{diag}(\mathbf{S}))^{1/2}$ is positive semidefinite and the nonnegative weighted integral preserves the convexity, \eqref{Dirichletaverage} is concave and thus \eqref{R-form_optimal_prob} is convex over $\mathbf{S}$. 
\end{proof}
Simplified with Lemma \ref{schurconvexofR}, \eqref{R-form_optimal_prob} becomes
\begin{equation}
\begin{aligned}
&\min_s s-\tilde{\sigma}_x\sqrt{2d}\frac{\Gamma({\frac{d+1}{2})}}{\Gamma({\frac{d}{2}+1)}}R_{\frac{1}{2}}(\frac{1}{2}\mathbf{1};s\cdot\mathbf{1})
=\min_s s-\tilde{\sigma}_x\sqrt{2d}\frac{\Gamma({\frac{d+1}{2})}}{\Gamma({\frac{d}{2}+1)}}s^{1/2},
\end{aligned}
\end{equation}
where the equality follows \cite[{19.16.11}]{822801}. Then, by solving the quadratic programming, the optimal  
\[
s=\frac{d}{2}\frac{\Gamma(\frac{d+1}{2})^2}{\Gamma(\frac{d}{2}+1)^2}\tilde{\sigma}_x^2
\]
with minimum value of \eqref{eq_GaussianManiAppr} when $q=2$ is
\begin{equation}\label{inequalitywhenDiieqal}
\tilde{\sigma}_x\sqrt{1-\frac{d}{2}\frac{\Gamma(\frac{d+1}{2})^2}{\Gamma(\frac{d}{2}+1)^2}}.
\end{equation}
By Lemma \ref{ThmGaussianMani}, the upper bound \eqref{ubvalueofslice} for \eqref{eq_SlicedWGAN} from \eqref{upperboundofslice} is obtained. From \cite[(5.11.12)]{822801}, we have the asymptotic equivalence (same notation as $X\sim\mathcal{N}$)
\[\frac{d}{2}\left(\Gamma(\frac{d+1}{2})/\Gamma(\frac{d}{2}+1)\right)^2\sim 1.
\]
The upper bound \eqref{ubvalueofslice} then goes to zero as $d\rightarrow \infty$, and the lower bound of \eqref{eq_SlicedWGAN} is obviously zero. Thus, the linear generator with the parameters \eqref{optimallinearparameter} is asymptotically optimal for the sliced WGAN \eqref{eq_SlicedWGAN}.

\subsection{Proof of Lemma \ref{inner_discriminator_unproj_equal_proj}} \label{proof_inner_discriminator_unproj}
The R.H.S. of \eqref{innerdiscriminatorequality} can be solved by following the proof steps in Sec. \ref{proofSimoid}. When $q=2$, for any $\omega$,
\begin{equation}\label{eq_OTproject}
\begin{aligned}
&\inf_{\pi_\omega\in \Pi_\omega(\mu_\omega,\nu^\Theta_\omega)} \mathbb{E}_{\pi_\omega}[|\omega^T X - \omega^T\Theta Z|^2]
=\int_\R|\omega^Tx-t^\Theta_\omega(\omega^Tx)|^2d\mu_\omega(\omega^Tx),
\end{aligned}
\end{equation}
where the optimal transport function 
\begin{equation}\label{t^theta}
t^\Theta_\omega(\omega^Tx)=\sup\left\{y\in\R:F_{\nu^\Theta_w}(y)\le F_{\mu_\omega}\left(\omega^T x \right)\right\}.
\end{equation}
For the L.H.S. of \eqref{innerdiscriminatorequality}, when $q=2$, for any $\omega$,
\begin{align}
&\inf_{\pi'\in \Pi'(\mu,\nu^\Theta_\omega)} \mathbb{E}_{\pi'}[|\omega^T X - \omega^T\Theta Z|^2]\label{eq_M21OT}
=\int_{\R^d}|\omega^Tx-t^\Theta(x)|^2d\mu(x),
\end{align}
where the optimal transport function $t^\Theta(x):\R^d\rightarrow\R$ has the same value as $t^\Theta_\omega(\omega^Tx)$ in \eqref{t^theta} given $\omega$, which will be proved in Lemma \ref{multitoonetransport} later. Then, with the change of variables\cite{Cantelli_inequality}, \eqref{eq_M21OT}and \eqref{eq_OTproject} are equal.

To show that the optimal transport function $t^\Theta(x)$ for \eqref{eq_M21OT} equals $t^\Theta_\omega(\omega^Tx)$ in \eqref{t^theta}, we first prove that \eqref{eq_M21OT} belongs to the general multi-to-one optimal transport problem in \cite{chiappori2017multi} by checking its properties. To do this, let $Y=\omega^TG_\Theta(Z)$ and introduce the following sets.
For each fixed $y\in\R$, for all $k\in\R$, we define sets
\begin{align}
X(y,k) &:= \left\{x\in \R^d : -\frac{\partial}{\partial y}c(x,y)=k\right\}=\left\{x\in \R^d : \omega^T x = y + {k}/{2}\right\},\notag\\
X_{\le}(y,k) &:= \left\{x\in \R^d : -\frac{\partial}{\partial y}c(x,y)\le k\right\}=\left\{x\in \R^d : \omega^T x \le y + {k}/{2}\right\},\label{eq_M21superlevel set}
\end{align}
where $c(x,y)=(\omega^Tx-y)^2$.

\begin{lemma}\label{prop:OT}
Suppose that $\mu$ vanishes on hypersurfaces in $\R^d$ and $\nu_w^\Theta$ has no atom. For multi-to-one-dimensional inner discriminator problem \eqref{eq_M21OT}, we have the following properties
\begin{itemize}
\item [(a)] For each $y\in\R$, there exists a maximal interval $K(y) = [k^-(y),k^+(y)]\neq\emptyset$ such that 
\begin{equation}\label{mu=nu}
\mu(X_\le(y,k)) = \nu_w^\Theta((-\infty,y))\quad \forall k\in K(y). 
\end{equation}
\item [(b)] $\pi = (id\times t^\Theta(x))_\# \mu$ is the unique minimizer of \eqref{eq_M21OT}, where $t^\Theta(x):\R^d\to\R$ is defined by
\begin{equation}\label{F}
t^\Theta(x) =y\quad\hbox{if}\quad x\in X(y,k^+(y)).
\end{equation}
\end{itemize}
\end{lemma}

\begin{proof}
(a) $c(x,y)$ is non-degenerate in the sense of \cite[Definition 3.1]{chiappori2017multi} since the rank of $H_c\ge1$ for all $(x,y)\in\R^d\times\R$, where $H_c$ is the Hessian matrix of $c(x,y)$. The desired result follows from \cite[Theorem 4(a)]{chiappori2017multi}. 
\begin{equation}\label{Hessian_C}\notag
H_c = \begin{bmatrix}
2ww^T & -2w \\
-2w^T & 2 
\end{bmatrix}.
\end{equation}

(b) To apply the optimal transport in \cite{chiappori2017multi}, we need to check whether our \eqref{eq_M21superlevel set} meets the requirements listed within. We first show that $f^{\pm}(y):=y\mapsto y+{k^\pm(y)}/{2}$ are non-decreasing. By contradiction, suppose that there exist $y<y'$ such that $f^+(y)>f^+(y')$. Observe from \eqref{eq_M21superlevel set} that  
\begin{equation}\label{X contains X}
\begin{aligned}
&X_\le(y,k^+(y)) = \{x\in \R^d : w^T x \le f^+(y)\} 
\supseteq  \{x\in \R^d : w^T x \le f^+(y')\} = X_\le(y',k^+(y')),
\end{aligned}
\end{equation}
which implies $\mu(X_\le(y,k^+(y)))\ge \mu(X_\le(y',k^+(y')))$. On the other hand, \eqref{mu=nu} entails 
\begin{equation}\label{mu<mu}
\begin{aligned}
&\mu(X_\le(y,k^+(y))) = \nu^\Theta_w((-\infty,y))
\le \nu^\Theta_w((-\infty,y')) = \mu(X_\le(y',k^+(y'))).
\end{aligned}
\end{equation}
We therefore conclude \[\mu(X_\le(y,k^+(y))) = \mu(X_\le(y',k^+(y')))\] and consequently
\begin{equation}\label{nu=nu}
\nu^\Theta_w((-\infty,y)) = \nu^\Theta_w((-\infty,y')).
\end{equation}
As $y+{k^+(y)}/{2}=f^+(y)>f^+(y') = y'+{k^+(y')}/{2}$, we can choose $\bar k > k^+(y')$ such that $y'+\bar k/2 = y+{k^+(y)}/{2}$. By \eqref{eq_M21superlevel set}, this implies $X_\le(y',\bar k) = X_\le(y,k^+(y))$. It then follows that 
\[
\begin{aligned}
\mu(X_\le(y',\bar k)) = \mu(X_\le(y,k^+(y))) 
=  \nu^\Theta_w((-\infty,y)) = \nu^\Theta_w((-\infty,y')),
\end{aligned}
\]
where the second equality follows from \eqref{mu=nu} and the third equality is due to \eqref{nu=nu}. This readily contradicts part (a), which asserts that $k^+(y')$ is the largest solution $k\in\R$ to $\mu(X_\le(y', k))= \nu^\Theta_w((-\infty,y'))$. A similar argument shows that $y\mapsto f^-(y)$ is also non-decreasing. 
As $y\mapsto y+{k^\pm(y)}/{2}$ are non-decreasing, we see from \eqref{eq_M21superlevel set} that 
\begin{equation}\label{cond. 1'}
\hbox{$y\mapsto X_\le(y,k^\pm(y))$ are non-decreasing}. 
\end{equation}
Now, consider $y<y'$ such that $\nu^\Theta_w((y,y'))>0$. With $\nu^\Theta_w((-\infty,y)) < \nu^\Theta_w((-\infty,y'))$, we deduce from \eqref{mu=nu} that $\mu(X_\le(y,k^+(y)))=\nu^\Theta_w((-\infty,y)) < \nu^\Theta_w((-\infty,y')) = \mu(X_\le(y',k^-(y')))$. This, together with \eqref{eq_M21superlevel set}, indicates that $y+k^+(y)/2 < y'+k^-(y')/2$. This in turn implies that 
\begin{equation}\label{cond. 2'}
X_\le(y,k^+(y))\subseteq X_<(y',k^-(y')).
\end{equation} 
Then, the desired result follows from \cite[Theorem 4(b)]{chiappori2017multi}.
\end{proof}

From Lemma \ref{prop:OT}(b), the optimal transport function in \eqref{eq_M21OT} has the closed-form solution \eqref{F}. We prove that this optimal transport function is well-defined by properties in Lemma \ref{prop:OT}.

\begin{lemma}\label{multitoonetransport}
For $\mu$ almost every ($\mu-a.e.) \ x\in\R^d$, $t^\Theta(x)=F_{\nu^\Theta_\omega}^{-1}(F_{\mu_\omega}(\omega^Tx))$ and can be equivalently expressed as the R.H.S of \eqref{t^theta}.
\end{lemma}
\begin{proof}
Note that $F_{\nu^\Theta_\omega}$ is strictly increasing on
\[
\R^{\nu^\Theta_\omega}:=\R\backslash\bigcup_{y<y',\ \nu^\Theta_\omega((y,y'))=0}[y,y').
\]
Hence, $F_{\nu^\Theta_\omega}^{-1}$ is well-defined on $\R^{\nu^\Theta_\omega}$. For any $y<y'$ such that $\nu^\Theta_\omega((y,y'))=0$, we have $\nu^\Theta_\omega((-\infty,y))=\nu^\Theta_\omega((-\infty,y'))$. By Lemma \ref{prop:OT}(a), this implies
\[
\mu\big(\{x\in\R^d: y+\frac{k^+(y)}{2}<\omega^Tx\le y'+\frac{k^+(y')}{2}\}\big)=0.
\]
Since $\mu$ vanishes on hypersurfaces, $\mu\big(\{x\in\R^d: \omega^Tx = y+\frac{k^+(y)}{2}\})$ for $(y,y')$ such that $\nu^\Theta_w((y,y')) = 0$ can be ignored. It follows that
\[
\mu\big(\{x\in\R^d:\omega^Tx= y+\frac{k^+(y)}{2}\text{for some $y\in\R^{\nu^\Theta_\omega}$}\}\big)=1.
\]
That is, for $\mu-a.e.\ x\in\R^d$, $t^\Theta(x)\in\R^{\nu^\Theta_\omega}$. Since we have the following relationship between CDFs by \eqref{eq_M21superlevel set}, \eqref{mu=nu} , and \eqref{F}, 
\begin{equation}\label{CDFrelation}
F_{\mu_\omega}\left(\omega^T x \right)=F_{\mu_\omega}\left(y+\frac{k^+(y)}{2} \right)=F_{\nu^\Theta_w}(y)=F_{\nu^\Theta_w}\left(t^\Theta(x) \right),
\end{equation}
then $t^\Theta(x)=F_{\nu^\Theta_\omega}^{-1}(F_{\mu_\omega}(\omega^Tx))$.
\end{proof}

\section{Empirical study with synthetic data} \label{sec:simu}
We have only considered population WGAN thus far. In practice, the distribution is estimated from training data and does not have a closed-form CDF. Therefore, it is of interest to empirically study the performance of the population WGAN result using synthetic training data. In Fig.~\ref{FigTheta2}, we consider the linear activation and plot the optimal $\theta_2^*$ in \eqref{eq_W2thetaOpt1}\eqref{optParaLinear1} obtained by solving population WGAN with $d=1, q=2$, and its estimate with synthetic data when $\mu$ is chosen to be (a) $\mathcal{N}(0,1)$ and (b) Laplace distribution \cite{LaplaceTIP} with mean $0$ and scale $1/\sqrt{2}$ (which also has a unit variance). Specifically, by generating training data $\{x^{(i)}\}^M_{i=1}$ according to the true distribution $\mu$, we empirically estimate $\theta_2^*$ from \eqref{eq_W2thetaOpt1} by
\begin{equation} \label{eq_empiricaltheta2}
\frac{1}{M}\sum^M_{i =1} x^{(i)} \Phi^{-1}(\hat{F}_\mu(x^{(i)}))
\end{equation}
and the kernel density estimation (KDE) \cite{jiang2017uniform} is used to replace the true CDF $F_\mu(x)$ with the estimated one $\hat{F}_\mu(x)$ from the training data. Details are provided in the Appendix \ref{kde}. For each distribution, our result shows a nice convergence behavior where the difference becomes smaller than 0.005 with only $M$=50000 sample size. The optimal $\theta_2^*$ in \eqref{eq_W2thetaOpt1}\eqref{optParaLinear1} is 1 and 0.98013 for Gaussian and Laplace distributions, respectively. Using a linear generator, the GAN output is also Gaussian, and thus the convergence behavior for Gaussian data is better than that for Laplace data. Note that $\theta^*_1$ in \eqref{eq_W2thetaOpt1} can be estimated by the sample mean, so the convergence is not shown in Fig.~\ref{FigTheta2}. We also use SGD with momentum to estimate $\theta_1$ and $\theta_2$ for  loss function \eqref{eq_transportGANE} under $d=1,q=1$, with gradient empirically obtained from \eqref{eq_W1condition} using KDE $\hat{F}_\mu(x)$. With $\mathcal{N}(\mu=1.5, \sigma^2=4)$ data, $(\theta_1, \theta_2)$ converges to $(1.4957, 2.0905)$ shown in Fig. \ref{localoptimalq1} and close to the population local optimum $(\mu, \sigma)$ for $q=1$. The convergence rate of the empirical WGAN solution to that of the population WGAN problem has been theoretically analyzed to be $M^{-2/d}$ for the LQG setting in \cite[Theorem 2]{Tse20}. Note that the convergence of iteratively solving empirical WGAN based on a good solver for \eqref{P} was given in \cite{chen2019gradual}. However, this convergence heavily relies on a good semi-discrete optimal transport solver, which is still hard to design empirically\cite{tacskesen2023semi}.
\begin{figure}[t]
                \centering
             \includegraphics[width=0.5\textwidth]{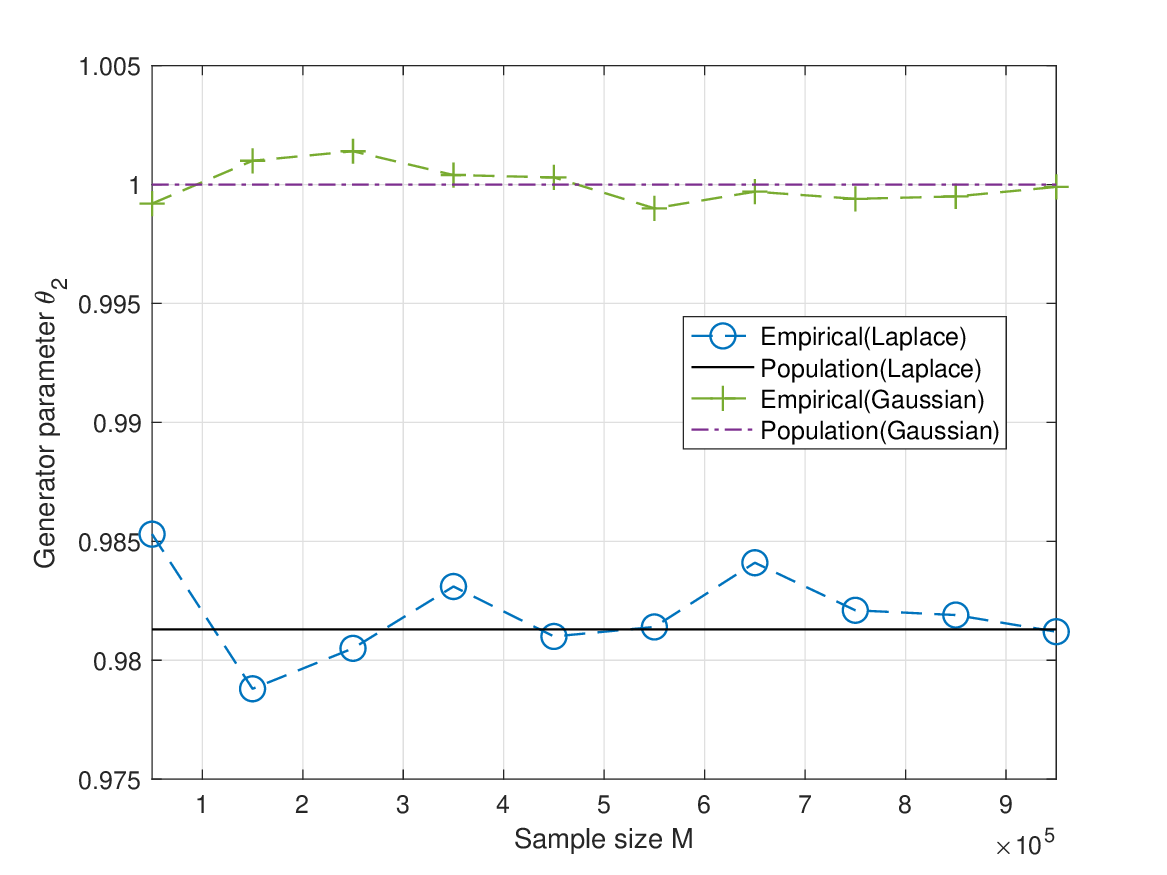}
             \caption{Comparison of optimal parameters in \eqref{eq_W2thetaOpt1}\eqref{optParaLinear1}  with their estimates \eqref{eq_empiricaltheta2} using synthetic data under $q=2, d=1$
             }
             \label{FigTheta2}
\end{figure}
\begin{figure}[t]
                \centering
             \includegraphics[width=0.5\textwidth]{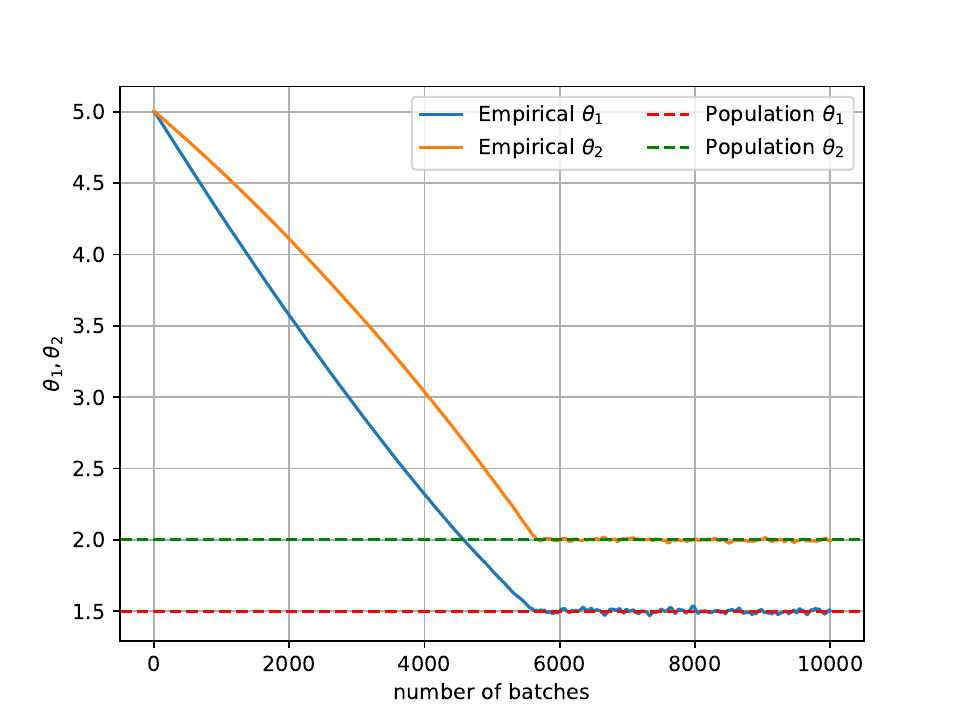}
             \caption{Convergence of $\theta_1$ and $\theta_2$ form SGD using \eqref{eq_transportGANE} for $q=1,d=1$
             }
             \label{localoptimalq1}
\end{figure}
\begin{figure}[t]
                \centering
             \includegraphics[width=0.5\textwidth]{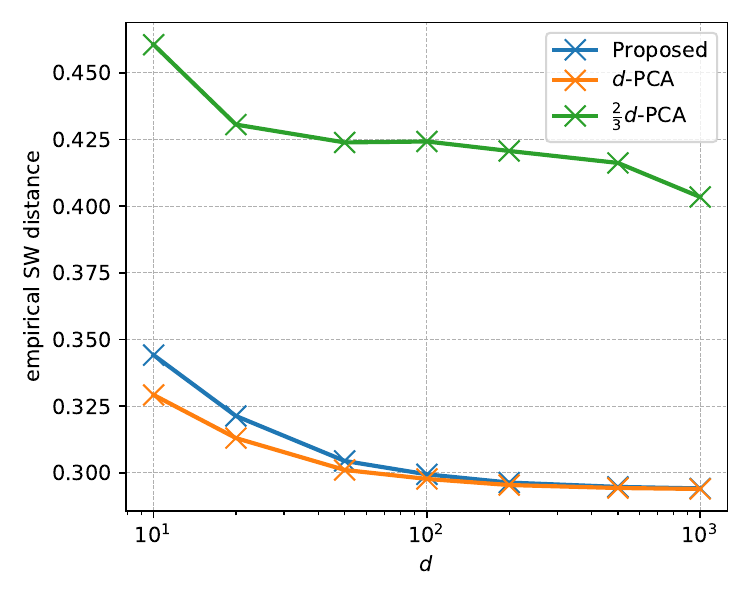}
             \caption{Comparison of sliced Wasserstein distance with our linear generator \eqref{optimallinearparameter} and that from r-PCA \cite{Tse20} when data is i.i.d. Laplace
             }
             \label{tse}
\end{figure}
\begin{figure}[t]
                \centering
             \includegraphics[width=0.5\textwidth]{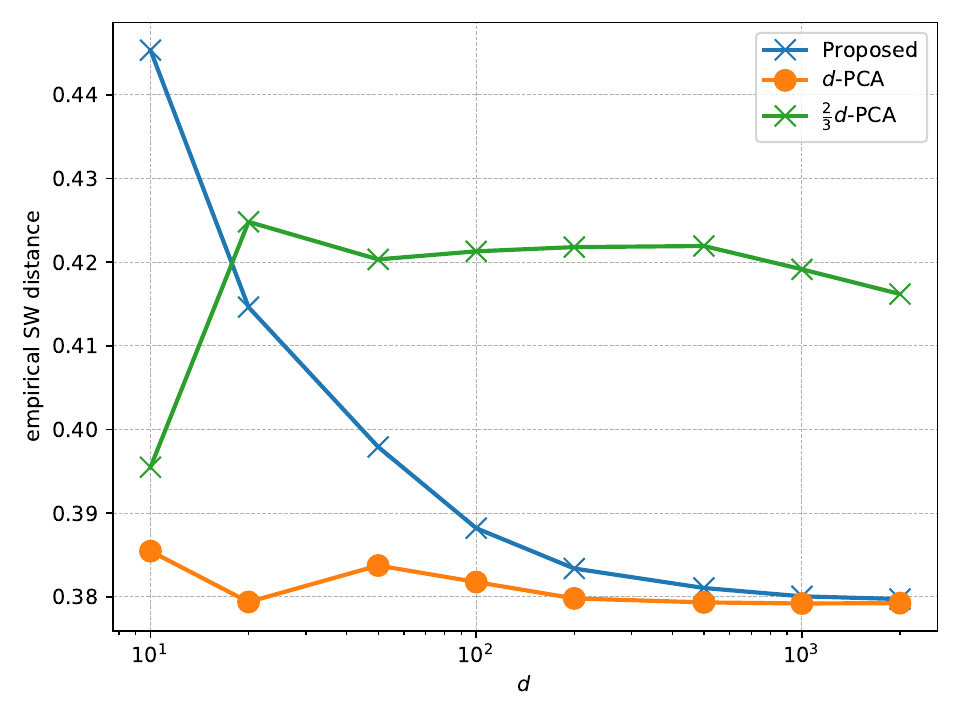}
             \caption{Comparison of sliced Wasserstein distance with our linear generator \eqref{optimallinearparameter} and that from r-PCA \cite{Tse20} when data is correlated AR model
             }
             \label{correlated_tse}
\end{figure}

For high-dimensional sliced WGAN with $d>1$, we compare our asymptotically optimal linear generator in Theorem \ref{optimalslicedWasserstein} with the $r$-PCA solution. The r-PCA is widely adopted \cite{Goodfellow-et-al-2016}, and is the optimal solution of classical population WGAN in the LQG setting \cite{Tse20}. First, we use synthetic data $\{X^{(i)}\}_{i=1}^M$ and $\{Z^{(i)}\}_{i=1}^M$ with sample size $M=4\times 10^5$, where $X^{(i)}$ generated from i.i.d Laplace distribution with zero mean and unit variance for all $i$ as \cite{nadjahi2021fast}, and $Z^{(i)}$ generated from i.i.d Gaussian distribution with zero mean and unit variance for all $i$. To solve the infimum of \eqref{eq_SlicedWGAN} in empirical, we sort $\{\omega_i^T X^{(j)}\}_{j=1}^M$ and $\{\omega_i^T\Theta Z^{(j)}\}_{j=1}^M$ as $[\omega_i^T X]_{(1)}\le[\omega_i^T X]_{(2)}\le\cdots\le[\omega_i^T X]_{(M)}$ and $[\omega_i^T\Theta Z]_{(1)}\le[\omega_i^T\Theta Z]_{(2)}\le\cdots\le[\omega_i^T\Theta Z]_{(M)}$ for each $i$. We compute the empirical sliced Wasserstein distance with Monte Carlo simulation based on $M_\omega=2\times10^4$ random projection samples from the unit sphere \eqref{unitsphere},
\begin{equation} \label{SWmonteappr}
\left(\frac{1}{M_\omega}\frac{1}{M}\sum_{\omega_i\in\Omega,\ i=1}^{M_\omega}\sum_{j=1}^M|[\omega_i^T X]_{(j)} - [\omega_i^T\Theta Z]_{(j)}|^2 \right)^{\frac{1}{2}}.
\end{equation}
We let $\Theta$ of \eqref{SWmonteappr} be our proposed linear generator parameters \eqref{optimallinearparameter} and the parameters of the $r$-PCA solution. In Fig. \ref{tse}, we can see that the sliced Wasserstein distance with our linear generator is much lower than that of $r$-PCA when $r=\frac{2}{3}d$.

When $r=d$, our results reach the same values as the $r$-PCA method at $d=1000$. The computational complexity of solving our proposed asymptotically optimal parameters is lower than the complexity of the $r$-PCA optimal parameters. For the $r$-PCA optimal parameters, the complexity of computing the covariance and finding the eigenvalues is $O((M+r)d^2)$, which is $O(Md^3)$ in the settings $r=\frac{2}{3}d$ and $r=d$. For our proposed linear generator parameters, we only need to compute $\E_\mu[\|X\|^2]$ whose complexity is $O(Md)$. Our asymptotically optimal linear generator has comparable performance to the r-PCA solution while reducing the computational resources required for parameters.

As \cite{nadjahi2021fast}, in addition to the i.i.d. synthetic data, we also use the data generated from the AR model in Remark \ref{remark:correlated}. For each sample $X^{(i)}$, the $j$-th element $X^{(i)}_j=0.5 X^{(i)}_{j-1}+\eta_j$ where $\eta_j$ is i.i.d. student's-t distribution noise. Given the correlated data, we compare the value of \eqref{SWmonteappr} with our proposed linear generator parameters to that of $r$-PCA. In Fig. \ref{correlated_tse}, except for very small $d=10$, we can see that the sliced Wasserstein distance with our proposed linear generator is lower than the i.i.d. data scenario when $r=\frac{2}{3}d$. When $r=d$, our results reach the same values as the $r$-PCA method at $d=2000$ with lower complexity than in the i.i.d. generated data.

\section{Conclusion}
In this work, we investigated the optimal solution characterization of population WGAN beyond the LQG setting. We first derived closed-form optimal parameters for one-dimensional WGAN when the NN has non-linear activation functions, and the data is non-Gaussian. For high‐dimensional data, we adopted the sliced Wasserstein framework and proved that linear generators can be asymptotically optimal. We further proposed an unprojected sliced WGAN that constrains the whole data marginal rather than the projected one in the original sliced WGAN, and we identified its asymptotic optimality as well. Empirical studies have shown that our generator for sliced WGAN can achieve better performance with linear complexity in the data dimension, compared to the celebrated r-PCA solution, which has cubic complexity.

\section*{Acknowledgement}
We would like to acknowledge the help of Wen-Yi Tseng and Kuan-Hui Lyu on the simulations.

\appendices

\section{Proof for Proposition \ref{prop:theta's}} \label{proofSimoid}
If $\Psi$ is continuous and strictly increasing and $\theta_2\in \R_+$, \eqref{quantile} can be expressed in closed-form as
\begin{equation}\label{quantile'}
t^\Theta(x) = \theta_1 + \theta_2 \Psi^{-1}(F_\mu(x)).
\end{equation}
To see this, since $\theta_2>0$, observe that
\begin{equation}
\begin{aligned}
&\Psi\bigg(\frac{t^\Theta(x)-\theta_1}{\theta_2}\bigg) =   \P\bigg(h(Z) \le \frac{t^\Theta(x)-\theta_1}{\theta_2}\bigg) \\ 
= &\P\left(G_\Theta(Z) \le t^\Theta(x)\right) = F_{\nu^\Theta}\left(t^\Theta(x)\right)\overset{(a)}= F_\mu(x),
\end{aligned}
\end{equation}
where the last equality follows from \eqref{quantile} and that $F_{\nu^\Theta}$ is continuous and strictly increasing; here,  $F_{\nu^\Theta}$ inherits the same properties from $\Psi$ thanks to \eqref{G}. Though the continuity of $\Psi$ is not needed in Lemma \ref{prop:1-d result} and its Kantorovich equivalence  \cite[Theorem 2.18]{villani2003topics}, without it, equality (a) will become an inequality, which harms finding closed-form $\Theta$. Then it follows that $\frac{t^\Theta(x)-\theta_1}{\theta_2}=\Psi^{-1}(F_\mu(x))$, which yields \eqref{quantile'}. On the other hand, if $\theta_2 = 0$, since $G_\Theta(Z)\equiv \theta_1\in\R$, we have $F_{\nu^\Theta}(y) = 1_{\{y\ge \theta_1\}}$. Plugging this into \eqref{quantile} directly gives $t^\Theta(x) = \theta_1$ for all $x\in\R$. This particularly shows that \eqref{quantile'} is also satisfied for the case $\theta_2=0$. Our target sub-problem becomes \eqref{GAN-sub}.
As will be explained below, the other sub-problem can be solved in a similar manner.

With closed-form representation \eqref{Wd_2}\eqref{quantile'} for the inner problem \eqref{P}, WGAN \eqref{eq_transportGANE} with $q=2,d=1$ can be simplified to be the following stochastic minimization problem
\begin{equation}\label{GAN'}
\min_{\theta_1 \in\R, \theta_2 \in\R_+}  \E_\mu\left[\left|X-\theta_1 - \theta_2 \Psi^{-1}(F_\mu(X))\right|^2 \right].
\end{equation}
Together with \eqref{G}, \eqref{GAN'} becomes the constrained optimization problem
\begin{align}
&\min_{\theta_1,\theta_2\in\R}\quad \! J(\theta_1,\theta_2) \!:= \! \int_\R \Big(\theta_1 \!+\! \theta_2 \Psi^{-1}(F_\mu(x))\!-\!x\Big)^2 \! F_\mu'(x)  dx \notag \\
&\hbox{subject to}\quad g(\theta_1,\theta_2) := -\theta_2\le 0. \label{constrained}
\end{align}
The corresponding first-order condition—namely, the KKT condition—is again in the form of \eqref{KKT1}-\eqref{KKT2}, where $\lambda\ge 0$ is the Lagrange multiplier.
By direct calculation, \eqref{KKT1} becomes
\begin{align*}
&\int_\R  \Big(\theta_1 + \theta_2 \Psi^{-1}(F_\mu(x))-x\Big)  F_\mu'(x) dx =0,\\
&\int_\R \Big(\theta_1 + \theta_2 \Psi^{-1}(F_\mu(x))-x\Big) \Psi^{-1}(F_\mu(x))F_\mu'(x)dx =\frac{\lambda}{2}.
\end{align*}
Recall $X$ is a random variable with CDF $F_\mu$, and the above equalities can be written as
\begin{align*}
&\theta_1 + \theta_2 \E\left[\Psi^{-1}(F_\mu(X))\right] = \E[X],\\
&\theta_1  \E\left[\Psi^{-1}(F_\mu(X))\right]  + \theta_2  \E\left[(\Psi^{-1}(F_\mu(X)))^2\right]  - \E\left[X \Psi^{-1}(F_\mu(X))\right]=  \lambda/2.
\end{align*}
Note that $F_\mu(X)\sim\operatorname{Uniform}[0,1]$, so that the CDF of $\Psi^{-1}(F_\mu(X))$ is simply $\Psi$. In other words, $\Psi^{-1}(F_\mu(X))$ and $h(Z)$ have identical distribution as \eqref{identical} in Sec.\ref{proofReLU}. The formulas above thus simplify to
\begin{align*}
&\theta_1 + \theta_2 \E\left[h(Z)\right]- \E[X] =0,\\
&\theta_1  \E\left[h(Z)\right]  + \theta_2  \E\left[(h(Z))^2\right] - \E\left[X \Psi^{-1}(F_\mu(X))\right] =  \frac{\lambda}{2}.
\end{align*}
Plugging $\theta_1 = \E[X] -  \theta_2 \E\left[h(Z)\right]$ from the first equality into the second one, we obtain
\[
 \E[X] \E\left[h(Z)\right] +\theta_2 \text{Var}(h(Z)) - \E\left[X \Psi^{-1}(F_\mu(X))\right] -\lambda/2=0.
\]
Hence, a solution $(\theta_1,\theta_2,\lambda)$ to \eqref{KKT1} must equivalently satisfy
\begin{equation}\label{KKT1'}
\begin{split}
\theta_1 &= \E[X] -  \theta_2 \E\left[h(Z)\right],\\
\lambda/2 &= \theta_2 \text{Var}(h(Z)) - \text{Cov}(X,\Psi^{-1}(F_\mu(X))).
\end{split}
\end{equation}

To solve the KKT condition \eqref{KKT1}-\eqref{KKT2} for candidate minimizers, we already know that \eqref{KKT1} boils down to \eqref{KKT1'}, while \eqref{KKT2} simply implies either $\lambda = 0$ or $\theta_2=0$. Also, because $\Psi^{-1}$ is strictly increasing and $F_\mu$ is nondecreasing, the map $x\mapsto \Psi^{-1}(F_\mu(x))$ is nondecreasing. This readily implies
\begin{equation} \label{eq_positiveCov}
\Cov\left(X, \Psi^{-1}(F_\mu(X))\right)\ge 0
\end{equation}
in \eqref{KKT1'} following the step of \eqref{eq_positiveCovReLU} in Sec.\ref{proofReLU}.
We therefore separate the proof into two cases from \eqref{eq_positiveCov}. \\
\noindent
{\bf Case I:} $\Cov(X,\Psi^{-1}(F_\mu(X)))> 0$. If $\theta_2 =0$, \eqref{KKT1'} entails $\lambda= -2\operatorname{Cov}(X,\Psi^{-1}(F_\mu(X)))<0$, which violates the requirement $\lambda\ge 0$ in \eqref{KKT1}-\eqref{KKT2}. If $\lambda=0$, solving \eqref{KKT1'} yields
\begin{equation}\label{theta's1}
\begin{split}
&\theta^*_1 =  \E[X] -  \frac{\Cov(X, \Psi^{-1}(F_\mu(X)))}{\Var(h(Z))}\E \left[h(Z)\right],\\ &\theta^*_2 =  \frac{\Cov(X, \Psi^{-1}(F_\mu(X))) }{\Var(h(Z)) }\ge 0.
\end{split}
\end{equation}
That is, the KKT condition gives a unique candidate optimizer $(\theta_1^*,\theta_2^*,\lambda^*)$, with $(\theta_1^*,\theta_2^*)$ as in \eqref{theta's1} and $\lambda^*=0$. To check the corresponding second-order condition, let $H_J$ and $H_g$ denote the Hessian matrices of $J(\theta_1,\theta_2)$ and $g(\theta_1,\theta_2)$, respectively. Clearly, $H_g = O_{2\times 2}$ and is all-zero, which implies
\begin{equation}\label{Hessian}
H_J + \lambda H_g =2
\begin{bmatrix}
1 & \E[h(Z)] \\
\E[h(Z)] & \E[(h(Z))^2]
\end{bmatrix}.
\end{equation}
As $\det(H_J+\lambda H_g)/4 = \E[(h(Z))^2]-\E[(h(Z))]^2=\text{Var}(h(Z))>0$, $H_J+\lambda H_g$ is positive definite. Hence, we conclude that $(\theta_1^*,\theta_2^*)$ in \eqref{theta's1} is the unique minimizer of \eqref{constrained}. \\
\noindent {\bf Case II:} $\Cov(X,\Psi^{-1}(F_\mu(X)))= 0$. As \eqref{KKT1'} entails $\lambda/2 = \theta_2 \text{Var}(h(Z))$, we have $\lambda=\theta_2 =0$ (by recalling that either $\lambda =0$ or $\theta_2 =0$). That is, the KKT condition gives a unique candidate optimizer $(\theta_1^*,\theta_2^*,\lambda^*) = (\E[X],0,0)$. Since $H_J+\lambda H_g$ is again given by \eqref{Hessian}, which is positive definite, we conclude that $(\theta_1^*,\theta_2^*) = (\E[X],0)$ is the unique minimizer of \eqref{constrained}.

\begin{lemma} \label{coro:min_value}
Suppose that $\mu$ has no atom, the CDF $\Psi$ of $h(Z)$ is continuous and strictly increasing, and $\Var(h(Z))>0$.
The sub-problem \eqref{GAN-sub} has a unique minimizer $\Theta^* = (\theta_1^*,\theta_2^*)\in\R\times\R_+$, defined as in \eqref{theta's}. Moreover, the minimum value is
\begin{equation}\label{min value}
\int_\R |x-t^{\Theta^*}(x)|^2 d\mu(x) = \Var(X) - \frac{\Cov(X,\Psi^{-1}(F_\mu(X)))}{\Var(h(Z))}.
\end{equation}
\end{lemma}

\begin{proof}
As \eqref{GAN-sub} is equivalent to \eqref{constrained}, Proposition~\ref{prop:theta's} directly shows that $\Theta^* = (\theta_1^*,\theta_2^*)\in\R\times\R_+$ as in \eqref{theta's} is the unique minimizer of \eqref{GAN-sub}. Thanks to \eqref{quantile'} and \eqref{theta's},
\begin{align*}
&\int_\R |x-t^{\Theta^*}(x)|^2 d\mu(x)\\
&= \int_\R \left(x-\theta_1^*-\theta^*_2 \Psi^{-1}(F_\mu(x))\right)^2 d\mu(x)\\
&= \int_\R \left(x-\E[X] - \theta^*_2 \Big(\Psi^{-1}(F_\mu(x))-\E[h(Z)]\Big)\right)^2 d\mu(x)\\
&= \int_\R (x-\E[X])^2 d\mu(x)- 2\theta^*_2 \int_\R (x-\E[X]) \Big(\Psi^{-1}(F_\mu(x))-\E[h(Z)]\Big) d\mu(x) \\
&\ \ + (\theta^*_2)^2 \int_\R \Big(\Psi^{-1}(F_\mu(x))-\E[h(Z)]\Big)^2 d\mu(x)\\
&= \Var(X)- 2\theta^*_2 \Cov(X, \Psi^{-1}(F_\mu(X))) + (\theta^*_2)^2 \Var(h(Z))\\
&=\Var(X) - \frac{\Cov(X,\Psi^{-1}(F_\mu(X)))^2}{\Var(h(Z))},
\end{align*}
where the fourth equality follows from the fact that $\Psi^{-1}(F_\mu(X))$ and $h(Z)$ have identical distribution, and the last equality is due to the definition of $\theta^*_2$ in \eqref{theta's}.
\end{proof}

\begin{lemma} \label{coro:GAN_sub}
Suppose that $\mu$ has no atom, the CDF $\Psi$ of $h(Z)$ is continuous and strictly increasing, and $\Var(h(Z))>0$.
The sub-problem \eqref{GAN-sub'} has a unique minimizer $\Theta^* = (\theta_1^*,\theta_2^*)\in\R\times\R_+$, defined by
\begin{equation}\label{theta's'}
\begin{split}
\theta^*_1 &=  \E[X] -  \frac{\Cov\left(X, \Psi^{-1}(1-F_\mu(X))\right) }{\Var(h(Z))}\E \left[h(Z)\right],\\
\theta^*_2 &=  \frac{\Cov\left(X, \Psi^{-1}(1-F_\mu(X))\right) }{\Var(h(Z)) }\le 0.
\end{split}
\end{equation}
Moreover, the minimum value is
\begin{equation}\label{min value'}
\int_\R |x-t^{\Theta^*}(x)|^2 d\mu(x) = \Var(X) - \frac{\Cov(X,\Psi^{-1}(1-F_\mu(X)))^2}{\Var(h(Z))}.
\end{equation}
\end{lemma}

\begin{proof}
Set $\eta:=-\theta_2$ and let $\Psi_0$ denote the CDF of $-h(Z)$. Given $x\in\R$, take $y = \Psi^{-1}_0(F_\mu(x))$. Observe that
$F_\mu(x) = \Psi_0(y) = \P(-h(Z)\le y) = \P(h(Z)\ge -y) = 1-\Psi(-y)$, which implies $y = - \Psi^{-1}(1-F_\mu(x))$. That is,
\begin{equation}\label{Psi_0}
\Psi^{-1}_0(F_\mu(x)) = - \Psi^{-1}(1-F_\mu(x))\quad \forall x\in\R.
\end{equation}
Now, since we can rewrite $G_\Theta$ in \eqref{G} as $G_\Theta =  \theta_1+\eta\cdot (-h(Z))$, the sub-problem \eqref{GAN-sub'} can be solved by Proposition~\ref{prop:theta's}, with $(\theta_2, h(Z), \Psi)$ therein replaced by $(\eta,-h(Z),\Psi_0)$. The unique minimizer $(\theta^*_1,\eta^*)$ is then given by
\begin{equation*}
\begin{split}
\eta^* &=  \frac{\Cov\left(X, \Psi^{-1}_0(F_\mu(X))\right) }{\Var(-h(Z)) } =  -\frac{\Cov\left(X, \Psi^{-1}(1-F_\mu(X))\right) }{\Var(h(Z)) }\\
\theta^*_1 &=  \E[X] -  \eta^* \E \left[-h(Z)\right] = \E[X] - \frac{\Cov\left(X, \Psi^{-1}(1-F_\mu(X))\right) }{\Var(h(Z)) } \E[h(Z)],
\end{split}
\end{equation*}
where we used \eqref{Psi_0} to change from $\Phi_0$ to $\Psi$. This readily yields \eqref{theta's'}. By replacing $\Psi$ in \eqref{min value} with $\Psi_0$ and using again \eqref{Psi_0}, we obtain the desired minimum value formula.
\end{proof}

Now we can prove Proposition \ref{prop:theta's}. By the linearity of covariances,
\begin{equation*}
\begin{split}
&\Cov\left(X, \Psi^{-1}(F_\mu(X))\right) + \Cov\left(X,\Psi^{-1}(1-F_\mu(X))\right)
 = \Cov\left(X, \Psi^{-1}(F_\mu(X))+\Psi^{-1}(1-F_\mu(X))\right).
\end{split}
\end{equation*}
If \eqref{Cov>0} holds, the above implies that the value in \eqref{min value} is smaller than or equal to that in \eqref{min value'}. Hence, the minimizer $(\theta_1^*,\theta_2^*)$ should be the one associated with the value \eqref{min value}, which is given by \eqref{theta's}. On the other hand, if \eqref{Cov>0} fails, the same argument shows that the value in \eqref{min value'} is smaller than that in \eqref{min value}. The minimizer $(\theta_1^*,\theta_2^*)$ should then be the one associated with the value \eqref{min value'}, which is given by \eqref{theta's'}.

\section {Proof of Lemma \ref{ReLUthetanegative}}\label{prooflReLUnegative}
Under the condition $\theta_2\le 0$, we deduce from \eqref{quantile} and \eqref{cdf ReLU} that
\begin{equation}\label{quantile' ReLU'}
t^\Theta(x) =
\begin{cases}
\theta_1- \theta_2 \Phi^{-1}(F_\mu(x)),\quad &\hbox{if}\ F_\mu(x)\le 1/2,\\
\theta_1,\quad &\hbox{if}\ F_\mu(x)> 1/2.
\end{cases}
\end{equation}
Indeed, for the case $F_\mu(x)\le 1/2$, with $\theta_2\le 0$, we get
\[
\begin{split}
\theta_1 +\theta_2 \Phi^{-1}(1-F_\mu(x)) = \theta_1 -\theta_2\Phi^{-1}(F_\mu(x)),
\end{split}
\]
where we used the fact $\Psi(v)=\Phi(v)$ for $v\ge 0$ (recall \eqref{cdf ReLU}) and that $Z$ (whose CDF is $\Phi$) is distributed symmetrically about 0.

In the following, we will often write $\theta_2\le 0$ in terms of $\eta:=-\theta_2\ge 0$, so that we can more easily borrow the arguments from $\theta_2\ge 0$. Similarly to \eqref{constrained ReLU}, the sub-problem \eqref{GAN-sub'} now becomes
\begin{equation}\label{constrained ReLU'}
\begin{split}
&\min_{\theta_1\in\R,\eta\ge 0}  J(\theta_1,\eta) :=  \int_{\{F_\mu(x)> 1/2\}} \Big(\theta_1 -x\Big)^2 F_\mu'(x) dx+\int_{\{F_\mu(x)\le 1/2\}} \Big(\theta_1 + \eta \Phi^{-1}(F_\mu(x))-x\Big)^2 F_\mu'(x) dx\\
&\hbox{subject to}\quad g(\theta_1,\eta) := -\eta\le 0.
\end{split}
\end{equation}
The corresponding first-order condition---namely, the KKT condition---is again in the form of \eqref{KKT1}-\eqref{KKT2} (with $\theta_2$ therein replaced now by $\eta$), where $\lambda\ge 0$ is the Lagrange multiplier. Now, by following the same arguments in \eqref{11}-\eqref{22'} (with $\theta_2$ replaced by $\eta$ and the two sets $\{F_\mu(x)\le 1/2\}$ and $\{F_\mu(x)> 1/2\}$ switched), we obtain $\theta_1 = \E[X] -  \eta \E\left[\Phi^{-1}(F_\mu(X)) 1_{\{F_\mu(X)\le 1/2\}}\right]$ and
\begin{align}
 &\E[X] \E\left[\Phi^{-1}(F_\mu(X)) 1_{\{F_\mu(X)\le 1/2\}} \right]
 +\eta \text{Var}\left(\Phi^{-1}(F_\mu(X))1_{\{F_\mu(X)\le 1/2\}}\right)\notag\\
&-\E\left[X \Phi^{-1}(F_\mu(X)) 1_{\{F_\mu(X)\le 1/2\}} \right]-\lambda/2=0. \label{22'''}
\end{align}
Similarly to \eqref{E calc} and \eqref{E^2 calc}, we have
\begin{align}
&\E\left[\Phi^{-1}(F_\mu(X)) 1_{\{F_\mu(X)\le 1/2\}} \right]
= \E[Z 1_{\{\Phi(Z)\le1/2\}}] = \E[Z 1_{\{ Z\le 0\}}] = -\frac{1}{\sqrt{2\pi}},\label{E calc'}\\
&\E\left[(\Phi^{-1}(F_\mu(X)))^2 1_{\{F_\mu(X)\le 1/2\}} \right] 
=  \E[Z^2 1_{\{\Phi(Z)\le1/2\}}] = \E[Z^2 1_{\{ Z\le 0\}}] =  \frac{1}{2}. \label{E^2 calc'}
\end{align}
In view of this, \eqref{11'} and \eqref{22'''} reduce to
\begin{equation}\label{KKT1'''}
\begin{split}
\theta_1 &= \E[X] +  \frac{\eta}{\sqrt{2\pi}},\\
\lambda/2 &= \frac{\eta}{2} \left(1-\frac{1}{\pi}\right) - \text{Cov}(X,\Phi^{-1}(F_\mu(X)) 1_{\{F_\mu(X)\le 1/2\}}).
\end{split}
\end{equation}

\begin{lemma}\label{prop:theta's ReLU'}
Suppose that $\mu$ has no atom. Then, \eqref{constrained ReLU'} has a unique minimizer $(\theta_1^*,\eta^*)\in\R\times\R_+$ given by
\begin{equation}\label{theta's ReLU'-pre}
\begin{split}
\theta^*_1 &=  \E[X]  + \frac{\sqrt{2\pi}}{\pi-1} \Cov\big(X, \Phi^{-1}(F_\mu(X)) 1_{\{F_\mu(X)\le 1/2\}}  \big),\\
\eta^* &=   \frac{2\pi}{\pi-1} \Cov\left(X, \Phi^{-1}(F_\mu(X)) 1_{\{F_\mu(X)\le 1/2\}} \right)\ge 0.
\end{split}
\end{equation}
\end{lemma}
\begin{proof}
Consider $f(x):=\Phi^{-1}(F_\mu(x)) 1_{\{F_\mu(x)\le 1/2\}}$ and set $x^*:=\sup\{x\in\R : F_\mu(x)\le 1/2\}$.
As $F_\mu$ is nondecreasing and $\Phi^{-1}$ is strictly increasing, for $x<x^*$, $f(x)= \Phi^{-1}(F_\mu(x))$ is nondecreasing with $f(x)\le\Phi^{-1}(1/2)=0$; for $x\ge x^*$, we simply have $f\equiv 0$. Since $f$ is nondecreasing on $\R$, an argument similar to \eqref{Cov nonn} implies $\Cov\left(X, \Phi^{-1}(F_\mu(X)) 1_{\{F_\mu(X)\le 1/2\}}  \right)\ge 0$.
Specifically, we deal with the two cases where $\Cov\left(X, \Phi^{-1}(F_\mu(X)) 1_{\{F_\mu(X)\le 1/2\}}  \right)$ is strictly positive and zero separately. In each case, we solve \eqref{KKT1'''} (in place of \eqref{KKT1''}) and find that $H_J+\lambda H_g$ now takes the form
\[
\begin{bmatrix}
1 &  -\frac{1}{\sqrt{2\pi}}\\
 -\frac{1}{\sqrt{2\pi}}& \frac12
 \end{bmatrix},
\]
which is positive definite.
\end{proof}
By repeating all the calculations in Lemma~\ref{coro:ReLU} (with the two sets $\{F_\mu(x)\le 1/2\}$ and $\{F_\mu(x)> 1/2\}$ switched and $\theta_2$ replaced by $\eta^*$ in \eqref{theta's ReLU'-pre}), we obtain the minimum value formula.

\section{Proof of Corollary \ref{prop:theta's Linear}}\label{proof:theta's Linear}

Besides checking \eqref{Cov>0}, it is easy to see that one can limit $\theta_2 \in \R_+$ such that the optimal parameter is \eqref{theta's} when $h(Z)=Z$. Now the distribution of $-h(Z)$ is still the same as $h(Z)$. Then one can rewrite \eqref{G} as $
G_\Theta(Z) = \theta_1+(-\theta_2)*(-Z),
$
and absorb the case for $\theta_2<0$ into that for $\theta_2 \geq 0$. Since now $h(Z)=Z, Z \sim \mathcal{N}(0,1)$, \eqref{theta's} reduces to
\begin{equation}
\begin{split}
\theta^*_2 &=  \Cov\left(X, \Phi^{-1}(F_\mu(X))\right)\\
\theta^*_1 &=  \E_\mu[X] ;
\end{split}
\end{equation}
Furthermore,
\begin{equation} \label{eq_transDataGaussian}
  \Phi^{-1}(F_\mu(X)) \stackrel{d}{=} Z,
\end{equation}
from $F_\mu(X) \stackrel{d}{=} \Phi(Z)$ and both uniformly distributed in
$[0,1]$ for continuous $X$. Thus $\E[\Phi^{-1}(F_\mu(X))]=0$ and we reach \eqref{eq_W2thetaOpt1}. The other method is that one can rewrite \eqref{eq_transportGANE} as
\begin{equation} \label{eq_transportGAN}
\min_{\Theta} \inf_{\pi\in \Pi(\mu,\nu)} \int
\|x-G_\Theta(z)\|^q d\pi(x,z).
\end{equation}
where $\nu$ is the distribution of the generator input, a Gaussian vector. The reason is that, by transforming random variables, the optimal solution of the inner discriminator problem of \eqref{eq_transportGANE} can be transformed to be the optimal solution of the inner discriminator problem in \eqref{eq_transportGAN}. Let $\pi^*(x,y,\Theta)$ be the solution of the inner discriminator problem, then from Danskin's theorem \cite[Proposition B.25]{alma991020283739704786}, its gradient over $\Theta$ is 
\begin{equation} \label{eq_gra_theta}
\int  \nabla_{\! \Theta} \|G_\Theta(z)-x\|^q d\pi^*(x,z,\Theta),
\end{equation}
Under $q=2$ with $d = 1$, the gradient over vector $\theta=(\theta_1, \theta_2)$ \begin{align}
& \int \nabla_{\! \Theta}
\left( x-(\theta_1+\theta_2z)\right)^2 d\pi^*(x,z,\theta)  \notag \\
=& -2\;[
 \E_\mu[X]-(\theta_1+\theta_2\E[Z]) \; ,
 \E_{\pi*}[XZ]-(\theta_1\E[Z]+\theta_2\E[Z^2]) \label{eqL_1dimW2gra}
]^T,
\end{align}
since the matrix derivative
\begin{align}
&\nabla_{\! \Theta}
\left( x-(\theta_1+\theta_2z)\right)^2 \notag \\
=&\nabla_{\! \Theta} \left(x - \begin{bmatrix} 1&z \end{bmatrix} \begin{bmatrix} \theta_1 \\   \theta_2 \end{bmatrix}\right)^T \left(x - \begin{bmatrix} 1&z \end{bmatrix} \begin{bmatrix} \theta_1 \\   \theta_2 \end{bmatrix}\right) \notag \\
=&-2 \begin{bmatrix} 1&z \end{bmatrix}^T  \left(x-\begin{bmatrix} 1&z \end{bmatrix} \begin{bmatrix}    \theta_1 \\   \theta_2 \end{bmatrix} \right)
\end{align}
Note that the optimal joint PDF in the LHS of \eqref{Wd_2} will make $G_\Theta(Z)=t^\Theta(X)$, which also holds for $\pi^*$ in \eqref{eq_gra_theta}. Together with \eqref{quantile'}\eqref{G}, $\pi^*$ will make
\[
\theta_1+\theta_2 \Phi^{-1}(F_\mu(X)) = \theta_1+\theta_2 Z,
\]
which results \eqref{eq_W2thetaOpt1} by letting \eqref{eqL_1dimW2gra} equals to $[0,0]^T$.

For \eqref{optParaLinear1} we swapping the role of $X$ and $G_\Theta(Z)$ in
Lemma \ref{prop:1-d result}, and from \eqref{eqL_1dimW2gra}
\begin{equation} \label{optParaLinear1'}
\theta^*_2=\E_{\pi*}[XZ]=\E_g[F_\mu^{-1}(\Phi(Z) \cdot Z].
\end{equation}
To see this, from Lemma \ref{prop:1-d result}
\begin{align*} \inf_{\pi \in
\Pi(\mu,\nu^\Theta)} \int_{\R \times \R} | x - y |^2 d\pi(x,y) =
\int_{\R} | t^{\Theta} (y) - y |^2 d\nu^{\Theta}(y),
\end{align*}
where $d\nu^{\theta}(y) $ is the distribution of Gaussian variable
\begin{equation} \label{eq_SimplyGZ}
Y = G_\Theta (Z) = \theta_1 + \theta_2 Z.
\end{equation}
The optimal transport is then
\[
t^\Theta(y) = F_{\mu}^{-1} \left( \Phi \left(
\frac{y-\theta_1}{\theta_2} \right) \right).
\]
Note that the optimal joint distribution $\pi*$ will equal that of $(t^\Theta(Y),Y)$. From \eqref{eq_SimplyGZ} and $t^\Theta(Y)=X$ we reach \eqref{optParaLinear1'}.

\section{Detailed Proof for Lemma \ref{ThmGaussianMani}} \label{ProofGaussianManifull}
First we prove that for $q=1,2$, the gap between \eqref{eq_GaussianManiProp} and \eqref{eq_GaussianManiAppr} is bounded as
\begin{align}
\Bigg| & \Big(\int_{\omega\in\Omega_G}\inf_{\pi_\omega} \mathbb{E}_{\pi_\omega}[|\omega^T X - \omega^TG_\Theta(Z)|^q] d\omega \Big)^{\frac{1}{q}}- 
\Big(\int_{\omega\in\Omega_G}\inf_\pi \mathbb{E}_\pi[|\hat{X}_G - \omega^TG_\Theta(Z)|^q] d\omega \Big)^{\frac{1}{q}} \Bigg| \leq f^\mu(d) \label{eq_GaussianManiApprgap}
\end{align}
where $f^\mu(d)$ is in \eqref{eq_GaussianManiBound}. To see this, first we prove
\[
\int_{\omega\in\Omega_G} W^q_q(\omega^T X, \hat{X}_G) d\omega \leq  (f^\mu(d))^q, q=1,2
\]
where as the RHS of \eqref{P} we define Wasserstein distance of order $q$ as
\[
W_q(\mu,\nu):= \left(\inf_{\pi\in \Pi(\mu,\nu)} \E_\pi[\|X-Y\|^q]\right)^{1/q}
\]
where the marginals with respect to the first and second variables are given by $\mu$ and $\nu$ respectively. The case for $q=2$ is proved by \cite[Corollary 3]{ReevesISIT17}. Now for $q=1$,
\begin{align*}
&\int_{\omega\in\Omega_G} W_1(\omega^T X, \hat{X}_G) d\omega  \le
\int_{\omega\in\Omega_G} W_2(\omega^T X, \hat{X}_G) d\omega  \le   \bigg(\int_{\omega\in\Omega_G} W^2_2(\omega^T X, \hat{X}_G) d\omega\bigg)^{1/2} \le f^\mu(d)
\end{align*}
where the first inequality is from $W_1\le W_2$ by the definition of Wasserstein distances (see \cite[Remark 6.6]{Villani-book-09}) and the second one is from Jensen's inequality. Then from triangle inequality, given $\omega$
\begin{align*}
&|W_q(\omega^T X, \omega^TG_\Theta(Z))-W_q(\hat{X}_G, \omega^TG_\Theta(Z))|  \leq  W_q(\omega^T X, \hat{X}_G), q=1,2,
\end{align*}
which implies
\begin{align}
&\int_{\omega\in\Omega_G} |W_q(\omega^T X, \omega^TG_\Theta(Z))-W_q(\hat{X}_G, \omega^TG_\Theta(Z))|^q d\omega \leq  (f^\mu(d))^q. \label{eq_sW1bound1}
\end{align}
Finally, by the triangle inequality
\begin{align*}
\Bigg|&\left(\int_{\omega\in\Omega_G} W^q_q(\omega^T X, \omega^TG_\Theta(Z))\right)^{1/q} \notag -\left(\int_{\omega\in\Omega_G} W^q_q(\hat{X}_G, \omega^TG_\Theta(Z))\right)^{1/q}\Bigg| \\ \leq & \left(\int_{\omega\in\Omega_G} |W_q(\omega^T X, \omega^TG_\Theta(Z))-W_q(\hat{X}_G, \omega^TG_\Theta(Z))|^q d\omega \right)^{1/q}.
\end{align*}
Note that when $q=1$, one does not need the triangle inequality to make the upper-bound valid. Then with \eqref{eq_sW1bound1}, we have \eqref{eq_GaussianManiApprgap}.

Next, now from $q=1$ in \eqref{eq_GaussianManiAppr} we aim to solve with the linear generator
\begin{equation} \label{eq_W1GaussianManiAppr0}
\min_\Theta \left(\int_{\omega\in\Omega_G}\inf_\pi \mathbb{E}_\pi[|\hat{X}_G - \omega^T\Theta Z|] d\omega \right),
\end{equation}
which equals to \eqref{eq_GaussianManiAppr20} with $q=1$ as follows.
Given $\Theta$ and $\omega$, following the proof of Corollary \ref{prop:W1d1}
\[
\inf_\pi \mathbb{E}_\pi[|\hat{X}_G - \omega^T\Theta Z|]=\frac{1}{\tilde{\sigma}_x}\mathbb{E}[|(\tilde{\sigma}_x-\sqrt{\omega^T\Theta\Theta^T\omega})\hat{X}_G|],
\]
since the optimal transport function will make the distribution of $\omega^T\Theta Z$ same as
\[
\frac{\sqrt{\omega^T\Theta\Theta^T\omega}}{\tilde{\sigma}_x} \hat{X}_G.
\]
Then solving \eqref{eq_W1GaussianManiAppr0} equals to solving
\begin{equation} \label{eq_W1GaussianManiAppr1}
\min_\Theta \E_\omega [ |\tilde{\sigma}_x-\sqrt{\omega^T\Theta\Theta^T\omega}| ]
\end{equation}
since
\[
\mathbb{E}[|(\tilde{\sigma}_x-\sqrt{\omega^T\Theta\Theta^T\omega})\hat{X}_G|]=|\tilde{\sigma}_x-\sqrt{\omega^T\Theta\Theta^T\omega}|\mathbb{E}[|\hat{X}_G|]. \]
Similarly, from $q=2$ in \eqref{eq_GaussianManiAppr} , we aim to solve
\begin{equation} \label{eq_W2GaussianManiAppr}
\min_\Theta \left(\int_{\omega\in\Omega_G}\inf_\pi \mathbb{E}_\pi[|\hat{X}_G - \omega^T\Theta Z|^2] d\omega \right)^{\frac{1}{2}},
\end{equation}
which equals solving
\begin{align} \label{eq_W2GaussianManiAppr3}
 \min_\Theta \int_{\omega\in\Omega_G} \left(\tilde{\sigma}_x-\sqrt{\omega^T\Theta\Theta^T\omega}\right)^2  d\omega
\end{align}
from Corollary \ref{prop:theta's Linear}. And \eqref{eq_W2GaussianManiAppr3} equals to \eqref{eq_GaussianManiAppr20} with $q=2$.

Finally, it is easy to see that \eqref{eq_GaussianManiAppr20} with $q=2$ equals \eqref{eq_W2GaussianManiAppr4}. Here, we focus on the term
\begin{equation} \label{eqSqrtQuadratic0}
\E_w [\sqrt{\omega^T\Theta\Theta^T\omega}]=\E_{U_w}[U_wU_w^{-1/2}]
\end{equation}
in \eqref{eq_W2GaussianManiAppr4} and prove that it equals to \eqref{eqSqrtQuadratic}. By recalling the gamma function $\Gamma(t):= \int_0^\infty y^{t-1} e^{-y}dy$ for all $t>0$, we observe that
\begin{align*}
U_w^{-1/2} &= U_w^{-1/2} \frac{1}{\Gamma(1/2)}\int_0^\infty y^{1/2-1} e^{-y}dy\\
&= \frac{1}{\Gamma(1/2) U_w}\int_0^\infty \bigg(\frac{y}{U_w}\bigg)^{1/2-1} e^{-y}dy
= \frac{1}{\Gamma(1/2)}\int_0^\infty z^{-1/2} e^{-zU_w}dz,
\end{align*}
where the second line follows from the change of variable $z = y/U_w$. Note that the above calculation is motivated by  \cite[Example 3.2b.2]{Mathai}. It follows that \eqref{eqSqrtQuadratic0} equals to
\begin{align*}
\frac{1}{\Gamma(1/2)} \int^\infty_0 z^{-1/2}\E_{U_w}[U_w\exp(-zU_w)]dz                                         
\end{align*}
which equals to \eqref{eqSqrtQuadratic} from $\E_{U_w}[U_w\exp(-zU_w)]=-\frac{\partial}{\partial z}M_{U_w}(-z)$ where we recall $M_{U_w}(z):=\E[e^{zU_w}]$ is the moment generating function of $U_w$. Moreover, from \cite[Theorem 3.2a.1]{Mathai}, $M_{U_w}(z)$ has a closed form as
\begin{equation} \label{eq_MGFQuadra}
M_{U_w}(z)=\left|\mathbf{I}-\frac{2 z}{d} \Theta\Theta^T\right|^{-1/2}.
\end{equation}
Plugging \eqref{eq_MGFQuadra} into \eqref{eqSqrtQuadratic}, together with \eqref{eq_W2GaussianManiAppr4}\eqref{eq_W2GaussianManiAppr5},  we have \eqref{eq_W2GaussianManiAppr2}.

\section{Solution for \eqref{eq_GaussianManiAppr} when $q=d=2$}\label{dimension2}
Now we use elliptic integrals to show the special case of Theorem \ref{optimalslicedWasserstein} and \eqref{ubvalueofslice} when $q=d=2$ as follows.
\begin{corollary} The minimum value of \eqref{eq_GaussianManiAppr} when $q=d=2$ is 
\begin{equation}\label{inequalitywhenDiieqal2}\notag
\tilde{\sigma}_x\sqrt{1-\frac{\pi}{4}},
\end{equation}
with the optimal parameters of the linear generator 
$\Theta=\mathbf{U}\Lambda \mathbf{V}^T$ where $\Lambda\Lambda^T=(\pi\tilde{\sigma}_x^2/4)\mathbf{I},$ and $\mathbf{U},\mathbf{V}^T$ are unitary.
\end{corollary}
\begin{proof} 
When $d=2$, the target function
\eqref{eq_W2GaussianDerivative} equals to
\begin{align}
&f(\mathbf{S})=\frac{S_{11}+S_{22}}{2}-\frac{2\tilde{\sigma}_x }{\Gamma(1/2)} \int^\infty_0 \frac{1}{2}z^{-1/2}\frac{S_{11}(1+zS_{22})+S_{22}(1+zS_{11})}{[(1+zS_{11})(1+zS_{22})]^{3/2}}dz . \label{quadraticcase2}
\end{align}
We first prove the Schur-convexity of \eqref{quadraticcase2} when $d=2$, 
Recalling the definition of Schur-convexity as follows, our target function 
$f(\mathbf{S}) : \{\mathbb{R}^+\cup0\}^2 \rightarrow  \mathbb{R} $ is Schur-convex if and only if $f$ is symmetric on $\{\mathbb{R}^+\cup0\}^2$, and
\begin{equation} \label{eq_Schurd2}
(S_{11}-S_{22})\left[\frac{\partial f(\mathbf{S})}{\partial S_{11}}-\frac{\partial f(\mathbf{S})}{\partial S_{22}}\right]\ge0.
\end{equation} 
To prove \eqref{eq_Schurd2}, we first find the closed-form of $f(\mathbf{S})$ using elliptic integrals as follows. Let us define
\[
K(m):=\int^{\frac{\pi}{2}}_0\frac{d\theta}{\sqrt{1-m\sin^2\theta}}
\]
which is the complete elliptic integral of the first kind, and
\[
E(m)=\int^{\frac{\pi}{2}}_0\sqrt{1-m\sin^2\theta}d\theta
\]
which is the complete elliptic integral of the second kind. When $D_{11}\neq D_{22}\ne0$, the integral in \eqref{quadraticcase2} becomes
\begin{align*}
&\int_0^\infty\frac{1}{2}z^{-\frac{1}{2}}\frac{S_{11}(1+zS_{22})}{[(1+zS_{11})(1+zS_{22})]^{\frac{3}{2}}}dz\\
\overset{(a)}=&\int_0^\infty\frac{S_{11}(1+x^2S_{22})}{[(1+x^2S_{11})(1+x^2S_{22})]^{\frac{3}{2}}}dx,\\
\overset{(b)}=&\sqrt{S_{11}}\int_0^{\frac{\pi}{2}}\frac{\sec^2\theta}{\sec^3\theta\sqrt{1+\frac{S_{22}}{S_{11}}\tan^2\theta}}d\theta,\\
=&\sqrt{S_{11}}\int_0^{\frac{\pi}{2}}\frac{\cos^2\theta}{\sqrt{\cos^2\theta+\frac{S_{22}}{S_{11}}\sin^2\theta}}d\theta\\
=&\sqrt{S_{11}}\int_0^{\frac{\pi}{2}}\frac{1-\sin^2\theta}{\sqrt{1+(\frac{S_{22}}{S_{11}}-1)\sin^2\theta}}d\theta\\
=&\sqrt{S_{11}}\frac{1}{1-\frac{S_{22}}{S_{11}}}\int_0^{\frac{\pi}{2}}\frac{-\frac{S_{22}}{S_{11}}+1-(1-\frac{S_{22}}{S_{11}})\sin^2\theta}{\sqrt{1-(1-\frac{S_{22}}{S_{11}})\sin^2\theta}}d\theta\\
=&\frac{-\sqrt{S_{11}}S_{22}}{S_{11}-S_{22}}K(1-\frac{S_{22}}{S_{11}})+\frac{\sqrt{S_{11}}S_{11}}{S_{11}-S_{22}}E(1-\frac{S_{22}}{S_{11}}),
\end{align*}
where the change of variable $x=z^2$ and $x=\frac{\tan\theta}{\sqrt{D_{11}}}$ are respectively applied to step (a) and (b). Substituting this results into \eqref{quadraticcase2},
\begin{equation}\label{solvedintegral2}
f(\mathbf{S})=\left\{
\begin{matrix}
    \frac{S_{11}+S_{22}}{2}-\frac{2\tilde{\sigma}_x }{\Gamma(1/2)}[\frac{-\sqrt{S_{11}}S_{22}}{S_{11}-S_{22}}K(1-\frac{S_{22}}{S_{11}})+\\
    \frac{\sqrt{S_{11}}S_{11}}{S_{11}-S_{22}}E(1-\frac{S_{22}}{S_{11}})+\frac{-\sqrt{S_{22}}S_{11}}{S_{22}-S_{11}}K(1-\frac{S_{11}}{S_{22}})+\\
    \frac{\sqrt{S_{22}}S_{22}}{S_{22}-S_{11}}E(1-\frac{S_{11}}{S_{22}})],\quad\text{if }S_{11}\ne S_{22}\ne0\\
    \\
    S-\frac{\tilde{\sigma}_x }{\Gamma(1/2)}\pi\sqrt{S},\quad\text{if }S_{11}=S_{22}=S\\
    \\
    \frac{S_{11}}{2}-\frac{2\tilde{\sigma}_x }{\Gamma(1/2)}\sqrt{S_{11}},\quad\text{if }S_{22}=0\\
    \\
    \frac{S_{22}}{2}-\frac{2\tilde{\sigma}_x }{\Gamma(1/2)}\sqrt{S_{22}},\quad\text{if }S_{11}=0
\end{matrix}
\right..
\end{equation}

It is easy to check \eqref{eq_Schurd2} is valid besides the first case in \eqref{solvedintegral2}. Now we prove that \eqref{eq_Schurd2} is also valid when $S_{11}\neq S_{22}\ne0$. Using \cite[equation 19.4.1 and equation 19.7.2]{822801}
\begin{equation}\label{conditionschur2}
\begin{aligned}
&(S_{11}-S_{22})[\frac{\partial f(S_{11},S_{22})}{\partial S_{11}}-\frac{\partial f(S_{11},S_{22})}{\partial S_{22}}]
=\frac{\tilde{\sigma}_x }{\Gamma(1/2)\sqrt{S_{22}}}[(S_{11}+S_{22})K(1-\frac{S_{11}}{S_{22}})-2S_{22}E(1-\frac{S_{11}}{S_{22}})].
\end{aligned}
\end{equation}
Let 
\[
m=1-\frac{S_{11}}{S_{22}}
\]
with $m<1$, then proving \eqref{conditionschur2} is non-negative is equivalent to proving
\[
(2-m)K(m)-2E(m)
\]
is non-negative. For $0\le m<1$, using \cite[equation 19.5.1 and equation 19.5.2]{822801},
\begin{equation} \label{positiveschur2}
\begin{aligned}
&(2-m)K(m)-2E(m)\\
=&\frac{\pi}{2}[(2-m)\sum_
{n=0}^\infty\frac{(\frac{1}{2})_n(\frac{1}{2})_n}{n!n!}m^n-2\sum_
{n=0}^\infty\frac{(-\frac{1}{2})_n(\frac{1}{2})_n}{n!n!}m^n]\\
=&\frac{\pi}{2}\sum_
{n=0}^\infty(2n(\frac{1}{2})_{n-1}-m(\frac{1}{2})_n)\frac{(\frac{1}{2})_n}{n!n!}m^n,
\end{aligned}
\end{equation}
where $(x)_n$ is Pochhammer's symbol
\[
(x)_n=\left\{
\begin{matrix}
x(x+1)(x+2)\cdots(x+n-1)\ \ if\ n=1, 2, 3,...\\
1\ \ if\ n=0
\end{matrix}
\right..
\]
When $n\ge1$, 
\[
(2n(\frac{1}{2})_{n-1}-m(\frac{1}{2})_n)\frac{(\frac{1}{2})_n}{n!n!}m^n>0, 
\]
but 
\[
(2n(\frac{1}{2})_{n-1}-m(\frac{1}{2})_n)\frac{(\frac{1}{2})_n}{n!n!}m^n = -m 
\]
at $n=0$. We sum \eqref{positiveschur2} from $n=0$ to $3$,
\begin{align*}
&\sum_
{n=0}^3(2n(\frac{1}{2})_{n-1}-m(\frac{1}{2})_n)\frac{(\frac{1}{2})_n}{n!n!}m^n =\frac{1}{8}m^2+\frac{7}{64}m^3+\frac{25}{256}m^4 >0
\end{align*}
and therefore, 
\[
(2-m)K(m)-2E(m)>0, \; \mbox{when} \; 0\le m<1.
\]
While if $m<0$, using the imaginary-modulus transformation of \cite[equation 19.7.5]{822801}, one can transform the problem equivalent to the one with $0\le m<1$. Specifically,
\begin{equation}\label{negativeschur2}
\begin{aligned}
&(2-m)K(m)-2E(m)\\
=&\frac{2-m}{\sqrt{1-m}}K(\frac{-m}{1-m})-2\sqrt{1-m}E(\frac{-m}{1-m})\\
=&\frac{1}{1-m'}[(2-m')K(m')-2E(m')]>0
\end{aligned}
\end{equation}
where
\[
0\le m'=\frac{-m}{1-m}<1.
\]
Hence, \eqref{quadraticcase2} is Schur-convex, and thus has a global minimum at 
\[
S_{ii}=\Lambda\Lambda^T=\frac{\pi\tilde{\sigma}_x^2}{4},\quad  i=1,2,
\]
and thus the minimum value of \eqref{eq_GaussianManiAppr} when $d=2$ is $\tilde{\sigma}_x\sqrt{1-\pi/4}$.

\end{proof}

\section{Estimate of the data CDF $\hat{F}_\mu(x)$ in \eqref{eq_empiricaltheta2}}\label{kde}
We use kernel density estimation to estimate the data PDF
\[
\frac{1}{|\mathcal{M}|} \sum^{|\mathcal{M}|}_{i=1} \frac{1}{h}k_i\left(\frac{x-x_i}{h}\right),
\]
where the kernel $k_i(.)$ can be selected as Gaussian or one
easy to perform the CDF inverse. The idea of kernel density estimation is a piecewise density estimation, and it converges to the PDF when $|\mathcal{M}|$ is large \cite{jiang2017uniform}. To be specific, we first rearrange the data in set $\mathcal{M}$ such that
\[
x_1 \leq x_2 \leq \ldots \leq x_{|\mathcal{M}|}.
\]
Then if $k_i(x)$ is Epanechnikov which is $(3/4)(1-x^2)$ if $x_i - h \leq x \leq x_i + h$ and 0 else, and by integration the total probability
\begin{align*}
 &\int^{\infty}_{-\infty} \frac{1}{|\mathcal{M}|} \sum^{|\mathcal{M}|}_{i=1} \frac{1}{h}k_i\left(\frac{x-x_i}{h}\right) dx \\&=  \frac{1}{|\mathcal{M}|} \sum_{i=1}^{|\mathcal{M}|} \int_{-1}^{1} \frac{3}{4} (1-y_i^2) dy_i\\
   & =  \frac{1}{|\mathcal{M}|}  \sum_{i=1}^{|\mathcal{M}|} \frac{3}{4}\left[ y_i-\frac{y_i^3}{3} \right]_{-1}^{1}  =  \frac{1}{|\mathcal{M}|} \sum_{i=1}^{|\mathcal{M}|} 1 = 1
\end{align*}
where the change of variable is $y_i = \frac{x - x_i}{h}, dx = h dy_i$. Then the CDF estimation is
\begin{align*}
&\hat{F}_\mu(x) =  \int^{x}_{-\infty} \frac{1}{|\mathcal{M}|} \sum^{|\mathcal{M}|}_{i=1} \frac{1}{h}k_i\left(\frac{x'-x_i}{h}\right) dx' \\ &= \frac{1}{|\mathcal{M}|} \left(\sum^{L^x}_{i=1}1+ \int_{-1}^{(x-x_{L^x+1})/h} \frac{3}{4} (1-y_L^2) dy_L \right)  \\ &= \frac{1}{|\mathcal{M}|}\left(L^x+\frac{3}{4}\left(\frac{x-x_{L^x+1}}{h}-\frac{(x-x_{L^x+1})^3}{3h^3}+1-\frac{1}{3}\right)\right) \\ &= \frac{1}{|\mathcal{M}|}\left(L^x+\frac{1}{4} \left( 2 + 3\left(\frac{x - x_{L^x+1}}{h}\right) - \left(\frac{x - x_{L^x+1}}{h}\right)^3 \right) \right)
\end{align*}
if $x_1-h<x<x_{|\mathcal{M}|}+h$, where $L^x+1$ is the index of unique data point $x_{L^x+1}$ such that
\[
x_{L^x+1}-h<x<x_{L^x+1}+h.
\]
If $x \geq x_{|\mathcal{M}|}+h$ then $\hat{F}_\mu(x)=1$ and if $x \leq x_1-h$ then $\hat{F}_\mu(x)=0$. Also, the inverse CDF
\[
\hat{F}^{-1}_\mu(p)=2\sin \left(\frac{1}{3}\arcsin\big(2\left(|\mathcal{M}|p-L^p\right)-1\big) \right)h + x_{L^p+1}
\]
where
\[
\frac{L^p}{|\mathcal{M}|}<p<\frac{L^p+1}{|\mathcal{M}|}.
\]

\bibliographystyle{IEEEtran}
\bibliography{IEEEabrv,GANref,ganref3}

\end{document}